\documentclass[journal]{IEEEtran}
\ifCLASSINFOpdf
\usepackage[pdftex]{graphicx}
\else
\fi
\usepackage[cmex10]{amsmath}
\usepackage{url, amsfonts}
\usepackage{cite}

\usepackage{amsthm, amssymb, dsfont, float} 
\begin{document}
\bibliographystyle{IEEEtran}

\newtheorem{lemma}{Lemma} 
\newtheorem{corollary}{Corollary} 
\newtheorem{theorem}{Theorem} 
\newtheorem{definition}{Definition}
\newtheorem{remark}{Remark}
\newtheorem{example}{Example}

\newcommand{\Sim}[2]{\langle #1, \,#2\rangle_\mathcal{P}}
\newcommand{\Dis}[2]{|| #1\,,\;\,#2 ||_\mathcal{P}}
\newcommand{\Inf}[2]{\mathcal{I}_\mathcal{P}(#1\,,\;#2)}
\newcommand{\inff}[2]{\mathfrak{I}_{#1}(#2)}
\newcommand{\X}{\mathcal{X}}
\newcommand{\Y}{\mathcal{Y}}
\newcommand{\D}{\mathfrak{D}}
\newcommand{\Prob}{\mathbb{P}}
\newcommand{\E}{\mathbb{E}}
\newcommand{\R}{\mathbb{R}}
\newcommand{\F}{\mathcal{F}}
\newcommand{\HH}{\mathcal{H}}
\newcommand{\1}{\mathbb{I}}
\newcommand{\LL}{\mathcal{L}}
\newcommand{\Z}{\mathcal{Z}}
\newcommand{\G}{\mathcal{G}}
\newcommand{\SSS}{\mathcal{S}}
\newcommand{\PP}{\mathbb{P}}
\newcommand{\RES}{\textbf{\textup{Ess}}}

\title{A Mathematical Theory of Learning}
\author{Ibrahim~Alabdulmohsin
\thanks{I. Alabdulmohsin is with the Computer, Electrical and Mathematical Sciences and Engineering (CEMSE) Division at King Abdullah University of Science \& Technology (KAUST), Thuwal, Saudi Arabia. e-mail: ibrahim.alabdulmohsin@kaust.edu.sa}
\thanks{}}
\maketitle

\begin{abstract} 
In this paper, a  mathematical theory of learning is proposed that has many parallels with information theory. We consider Vapnik's General Setting of Learning in which the  learning process is defined to be the act of selecting a hypothesis in response to a given training set. Such hypothesis can, for example, be a decision boundary in classification, a set of  centroids in clustering, or a set of frequent item-sets in association rule mining. Depending on the hypothesis space and how the final hypothesis is selected, we show that a learning process can be assigned a numeric score, called \emph{learning capacity}, which is analogous to Shannon's channel capacity and satisfies similar interesting properties as well such as the data-processing inequality and the information-cannot-hurt inequality.  In addition, learning capacity provides the tightest possible bound on the difference between true risk and empirical risk of the learning process for all  loss functions that are parametrized by the chosen hypothesis. It is also shown that the notion of learning capacity equivalently quantifies how sensitive the choice of the final hypothesis is to a small perturbation in the training set. Consequently, algorithmic stability is both necessary and sufficient for generalization. While the theory does not rely on concentration inequalities, we finally show that analogs to classical results in learning theory using the Probably Approximately Correct (PAC) model can be immediately deduced using this theory, and conclude with information-theoretic bounds to learning capacity. 
\end{abstract} 

\section{Introduction} 
In this paper, we consider the General Setting of Learning introduced by Vapnik \cite{vapnik1999overview} in which a \emph{learning machine} $\LL$ is presented with a training set $S_m = \{Z_1,\ldots,\,Z_m\}\in\Z^m$ whose $m$ training examples $Z_i$ are drawn i.i.d. from a fixed unknown distribution $\PP(z)$. The task of the learning machine is to pick a hypothesis $H$ out of a pre-defined hypothesis space $\HH^{(m)}$ in order to \lq\lq summarize'' or \lq\lq fit'' the training set. In general, we assume that the observation space $\Z$ may or may not be numerical and that the hypothesis space $\HH^{(m)}$ can vary according to $m$. 

For example, in a binary classification task, a classifier may be presented with feature-label pairs $Z=(X,\,Y)\in\X\times\{-1,\,+1\}$ and the goal is to choose a function $f(x):\X\to\{-1,\,+1\}$ that can accurately predict the unknown label $Y$ given $X$. Here, any choice of $f(x)$ serves as an instance of $H$. In mean estimation, such as when we would like to predict $X\in\mathbb{R}^n$ using its expected value $\E[X]$, the hypothesis $H$ can be the empirical average of training examples whose space $\HH^{(m)}$ is the entire plane $\mathbb{R}^n$. In the latter case, $H$ is a deterministic function of the training set $S_m$. Finding weights in neural networks, prototypes in clustering methods, enclosing spheres in some outlier detection algorithms, and frequent itemsets in association rule mining are only a handful of examples to learning tasks that fall under such general learning setting. 

In the general learning setting, the learning machine $\LL$ picks an instance of $H$, denoted $h\in\HH^{(m)}$, according to some \emph{fixed} learning process. We model such  learning process as a probability distribution $H\sim\PP^{(m)}(h\,|\,S_m)$. Again, the process generally depends on the number of training examples. For example, if observations are binary $Z\in\{0,\,1\}$ and $\LL$ generates the sum of observations $H = \sum_{i=1}^m\, Z_i$, then $\HH^{(m)} = \{0,1,\ldots,\,m\}$. In the latter case, the probability of picking a specific hypothesis $\PP(H=h)$ given an instance of a training set $S_m=s_m$ is a degenerate distribution; the probability is one if $h$ is the sum of all training examples in $s_m$ and is zero otherwise. In general, however, the hypothesis $H$ might be a random function of $S_m$. Once a hypothesis is inferred, the learning process is concluded. 

In order to assess \emph{quality} of the inference, we assume a non-negative loss function with bounded range exists $L_H(Z):\,\Z\to [0,\,1]$, that is conditionally independent of the training set $S_m$ given $H$. That is, we assume that the Markov chain $S_m\to H\to L_H(Z)$ always hold. Formally, $L_H(Z)$ is a function of \emph{both} the inferred hypothesis $H$ and the observation $Z$. For example, an observation in SVM is a pair of features plus class label $Z=(X,\,Y)\in\mathbb{R}^n\times\{-1,\,+1\}$. Here, the hypothesis $H$ can be a separating hyperplane, i.e. $H$ takes its values from $\{(w,\,b)\,|\,w\in\mathbb{R}^n,\,b\in\mathbb{R}\}$ for some normal vector $w$ and offset $b$. In addition, one possible loss function is given by $L_H(Z)=L_{w,b}(x,\,y)=\1\{y(w^T\,x-b)\le 0\}$\footnote{Here, $\1\{\cdot\}$ is a 0-1 Boolean indicator.}, which is conditionally independent of the original training set $S_m$ given the inferred hypothesis $H=(w,\,b)$.

Having a loss function $L_H(Z)$ at hand, we define the \emph{true risk} of a hypothesis $H$ with respect to $L_H(Z)$ by the risk functional: 
\begin{equation}\label{true_risk} 
\hat R(H) = \E_{Z\sim\PP(z)}\;\big[L_H(Z) \big]
\end{equation}
Here, the true risk of any instance of $H$ is the expected value of $L_H(Z)$, where expectation is taken over $Z\sim\PP(z)$. We define the \emph{risk of the learning machine} $\LL$ with respect to $L_H(Z)$, denoted $R(\LL)$, to be the expected risk of its inferred hypothesis, where expectation is taken over all possible training sets and over all possible hypotheses. Precisely, we have: 
\begin{equation}\label{true_risk_learn} 
R(\LL) = \E_{S_m} \E_{H|S_m}\;\hat R(H)
\end{equation}

The \emph{ideal} final goal is to be able to obtain an unbiased estimator to the \emph{true risk} of a learning machine $R(\LL)$ given that we know its inference process. This allows us to quantify quality of the inference. One convenient estimator is the empirical loss, which for a fixed training set $S_m=s_m$ and a fixed hypothesis $H=h$ is defined as: 
\begin{equation}
R_{emp}(h,\,s_m) = \frac{1}{m}\,\sum_{z_i\in S_m}\, L_h(z_i)
\end{equation} 
Note in the above expression that $L_h(z)$ is implicitly a function of $h$. Analogously, we define the empirical risk of the learning machine $\LL$ by the expected empirical risk of its inferred hypothesis: 
\begin{equation}\label{emp_risk_eq}
R_{emp}(\LL) = \E_{S_m}\,\E_{H|S_m}\,R_{emp}(H,\,S_m)
\end{equation}  
An unbiased estimator to $R_{emp}(\LL)$ is usually available since both training examples and the inferred hypothesis are often known. Unfortunately, however, it has long been established that the empirical risk is a \emph{biased} (optimistic) estimator to the true risk. Hence, we desire to bound the difference $|R(\LL)-R_{emp}(\LL)|$ analytically in order to be able to correct for such bias. Such bound would hopefully shed some insight into the many phenomena associated with learning such as overfitting, underfitting, and the importance of algorithmic stability. 

However, quantifying the difference between true risk and empirical risk is quite subtle and several methods have been proposed in the past to answer it including uniform convergence, stability, Rademacher and Gaussian complexities, generic chaining bounds, the PAC-Bayesian framework, and  robustness-based analysis \cite{vapnik1999overview,blumer1989learnability,talagrand1996majorizing,mcallester1999some,mcallester2003pac,bousquet2002stability,bartlett2002rademacher,audibert2007combining,xu2012robustness}. Moreover, extensions to the semi-supervised setting have been proposed as well \cite{pacsemisupervised2005}. Concentration inequalities form the building blocks of such rich theories.

In this paper, a new approach of bounding the difference between true risk and empirical risk is introduced. Unlike earlier approaches, the mathematical theory presented here does not treat such difference as a problem of convergence of random variables to their expectations. In fact, we will show that even though observations $Z$ are always assumed to be drawn i.i.d. from the same underlying distribution $\PP(z)$, both in the past and into the future, true and empirical risks of a learning machine have different distributions because {the process of learning changes our \emph{posterior} distribution of the training set}. We will see that the theory can be confirmed numerically quite readily, and that it is rich enough to capture some of deepest aspects of  learning  even though we are only dealing with averages.

As will be demonstrated throughout the sequel, the learning theory proposed in this paper has many parallels with information theory. In particular, a notion of \emph{mutual affinity} in learning is closely related to mutual information, and the notion of \emph{learning capacity} is quite analogous to the  capacity of communication channels. In addition, important inequalities in information theory such as the \lq\lq data-processing'' inequality and the \lq\lq information-cannot-hurt'' inequality \cite{cover2012elements} have analogs within the theory of learning. The asymptotic equipartition property (AEP) plays a key role in both theories as well. In fact, we will also be able to derive information-theoretic bounds, which are close in spirit to the PAC-Bayesian bounds \cite{mcallester1999some,mcallester2003pac}.  

The rest of the paper is structured as follows. We  will first introduce fundamental concepts such as similarity and distance between probability distributions, and present the main theorems after that. We will introduce the notion of \emph{learning capacity}, which provides the tightest possible bound on the difference between true and empirical risks of  learning machines. After that, we present several interpretations of learning capacity. For instance, it will be shown that learning machines admit a partial order with interesting implications, and that capacity and stability of learning machines are two different faces of the same coin. Finally, we show connections between learning capacity and the effective support set size of observations as well as size of the hypothesis space. 

\section{Notation}\label{sec::notation} 
We will always use $\LL$ to refer to a learning machine, which is a formal specification of a learning process and it comprises of three components: 
\begin{enumerate}
\item The observation space $\Z$.
\item A sequence of hypothesis spaces $\HH^{(m)}$ for all $m\ge 1$.
\item A sequence of probability distributions $\PP^{(m)}(H\,|\,S_m)$ for all $m\ge 1$ and all $S_m\in\Z^m$, where $H\in\HH^{(m)}$. 
\end{enumerate} 
Formally, $\LL$ is a tuple:
\begin{equation*} 
\LL = \big(\Z,\;\{\HH^{(m)}\}_{m=1,2,\ldots},\;\{\PP^{(m)}(H\,|\,S_m)\}_{m=1,2,\ldots}\big)
\end{equation*}
Given a learning machine $\LL$, we interpret it by saying that for any $m$ i.i.d. observations $S_m=\{Z_1,\ldots,\,Z_m\}\in\Z^m$ received, a hypothesis $H\in\HH^{(m)}$ is generated randomly according to $\PP^{(m)}(h\,|\,S_m)$. In statistical terms, $H$ can be \emph{any summary statistic} of $S_m$. 

For example, if $\Z\subseteq\mathbb{R}$ and $\HH^{(m)}=\mathbb{R}$, then $H$ can be the mean, the maximum, the median, or any individual training example pick out of $S_m$. In fact, $H$ can also be entirely independent of $S_m$. If $\HH^{(m)} = \Z^m$ and $H=S_m$, we will say that $\LL$ is  a \emph{lazy learner}. A lazy learner receives a training set $S_m$ and returns the training set itself as a hypothesis, hence the name. 

If $X\sim\PP(x)$ is a random variable drawn from the alphabet $\X$ and $f(X)$ is a function of $X$, we write $\E_{X\sim\PP(x)}\,f(X)$ to mean $\sum_{x\in\X}\,\PP(x)\,f(x)$. Often, we will simply write $\E_X\,f(X)$ to mean $\E_{X\sim\PP(x)}\,f(X)$ if the probability distribution $\PP(x)$ is clear from context. If $X$ takes its values from a finite set $S$ uniformly at random, we write $\E_{X\sim S}\,f(X)$ to mean $\frac{1}{|S|}\sum_{x\in S}\,f(x)$. 

In general, random variables will be denoted using capital letters, and instances of random variables will be denoted using small letters. Finally, alphabets are denoted using calligraphic typeface. We will generally restrict attention to the case in which the observation space $\Z$ and the hypothesis space $\HH^{(m)}$ are finite or countably infinite, in which case $\PP(z)$ and $\PP^{(m)}(H\,|\,S_m)$ are probability mass functions, but the main results can be readily generalized.

Finally, we will denote the 0-1 indicator function using $\1\{\cdot\}$. If $X$ is a boolean random variable, then $\1\{X\}=1$ if and only if $X$ is \emph{true}, otherwise $\1\{X\}=0$.  

\section{Fundamental Concepts} 
\subsection{Similarity and Distance}
Our first fundamental concept is the notion of similarity and distance between two probability distributions:

\begin{definition} \label{ProbSimDist} 
Given two probability distributions $\PP_1(a)$ and  $\PP_2(a)$ defined on the same alphabet $\mathcal{A}$, we define similarity using $\Sim{\PP_1}{\PP_2} = \sum_{a\in\mathcal{A}} \min\{\PP_1(a), \,\PP_2(a)\}$. Also, we write $\Dis{\PP_1}{\PP_2} =1-\Sim{\PP_1}{\PP_2}$ to denote the total variation distance. 
\end{definition}

Intuitively speaking, similarity is a measure of the \emph{intersection} or \emph{overlap} between two probability distributions, and is sometimes referred to as \emph{the overlapping coefficient} \cite{reiser1999confidence}. In addition $\Dis{p}{q}$ is {the total variation distance but we will call it distance for simplicity.

Needless to say, the notion of similarity and distance is not analytic. Nevertheless, the following lemma reveals that $\Sim{\cdot}{\cdot}$ and $\Dis{\cdot}{\cdot}$ are, in fact, the infinite limit of a sequence of smooth analytic functionals of probability distributions. 

\begin{lemma}\label{SeriesExpansionLemma}
Let $\PP_1(a)$ and $\PP_2(a)$ be two probability distributions defined on the same alphabet $\mathcal{A}$. Then: 
\begin{equation}
\Dis{\PP_1}{\PP_2} = \prod_{t=1}^\infty \big(1- \sum_{a\in\mathcal{A}}\;\rho^{(t)}(a)\cdot\nu^{(t)}(a)\big)
\end{equation}
Here, $\rho^{(t)}(a)$ and $\nu^{(t)}(a)$ are probability distributions given by the following recursive definition:
\begin{align*}
\rho^{(t)}(a) &= \frac{\rho^{(t-1)}(a)\cdot(1\;-\;\nu^{(t-1)}(a))}{1-\sum_{b\in\mathcal{A}}\;\rho^{(t-1)}(b)\cdot\nu^{(t-1)}(b)},\quad \rho^{(1)}(a) = \PP_1(a)\\
\nu^{(t)}(a) &=  \frac{\nu^{(t-1)}(a)\cdot(1\;-\;\rho^{(t-1)}(a))}{1-\sum_{b\in\mathcal{A}}\;\rho^{(t-1)}(b)\cdot\nu^{(t-1)}(b)},\quad \nu^{(1)}(a) = \PP_2(a)
\end{align*}
\end{lemma}
\begin{proof} 
Writing $\PP_1(a)=\PP_1(a)\,\PP_2(a)\, + \PP_1(a)\,\big(1-\PP_2(a)\big)$, we have: 
\begin{equation*}
\PP_1(a)-\PP_2(a) = \big(1-\sum_{b\in\mathcal{A}}\;\PP_1(b)\cdot\PP_2(b)\big)\cdot \big(\rho^{(2)}-\nu^{(2)} \big)
\end{equation*}
Taking the 1-norm of both sides and using $\Dis{\PP_1}{\PP_2}=\frac{1}{2}\,||\PP_1-\PP_2||_1$, where the subtraction is element-wise, gives us: 
\begin{align*}
\Dis{\PP_1}{\PP_2} = \big(1-\sum_{b\in\mathcal{A}}\;\PP_1(b)\cdot\PP_2(b)\big)\cdot\Dis{\rho^{(2)}}{\nu^{(2)}}
\end{align*}
Repeating this process indefinitely on the right-hand side yields statement of the lemma.
\end{proof}

Of particular importance to us in the above lemma is the following bound, which will be useful later when we discuss algorithmic stability: 
\begin{equation}\label{distUpperBound}
\Dis{\PP_1}{\PP_2} \le \prod_{t=1}^T \big(1-\sum_{a\in\mathcal{A}}\;\rho^{(t)}(a)\cdot\nu^{(t)}(a)\big), \quad \text{for any $T\ge 1$}
\end{equation}

\begin{example}\label{distance_series_example}
Suppose we have two Bernoulli distributions $\PP_1(z)=(s,\,1-s)$ with probability of success $s$, and $\PP_2(z)=(\frac{1}{2}, \,\frac{1}{2})$ with probability of success $\frac{1}{2}$. Their distance is $\Dis{\PP_1}{\PP_2}=|s-\frac{1}{2}|$. The first few distributions $\rho^{(t)}$ and $\nu^{(t)}$ of Lemma \ref{SeriesExpansionLemma} are given by:
\begin{alignat*}{2}
\rho^{(1)} &= (s,\; 1-s)\,, &\quad \nu^{(1)}&=(\frac{1}{2},\; \frac{1}{2})\\
\rho^{(2)} &= (s,\; 1-s)\,, &\quad \nu^{(2)}&=(1-s,\; s)\\
\rho^{(3)} &= \Big(\frac{s^2}{s^2+(1-s)^2},\; \frac{(1-s)^2}{s^2+(1-s)^2}\Big)\\ \nu^{(3)}&=\Big(\frac{(1-s)^2}{s^2+(1-s)^2},\;\frac{s^2}{s^2+(1-s)^2}\Big)\,\\
\end{alignat*}
Figure \ref{ExmBernDisFig} shows first three approximations for $T=1,\,2,\,3$. Clearly, the approximation does approach the true distance $\Dis{\PP_1}{\PP_2}=|s-\frac{1}{2}|$ as expected. 
\begin{figure}[t]
\centering
\includegraphics[scale=0.21]{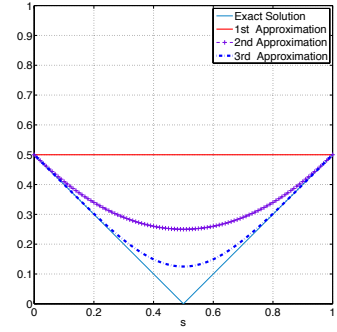}
\caption{First three approximations of Lemma \ref{SeriesExpansionLemma} are plotted against true distance for Example \ref{distance_series_example}.}
\label{ExmBernDisFig}
\end{figure}
\end{example} 

\begin{example}
To gain further insight into the role of total variation distance in learning problems, suppose we have a binary classification problem with features $X\in\X$ and labels $Y\in\{0,\,1\}$. Let $\PP_k(x)=\Prob(X=x|Y=k)$ be the class conditional distribution of features $X$. Define $\kappa = \max\{\Prob(Y=0),\,\Prob(Y=1)\}$. Then, the optimal Bayes rate $e^*$ satisfies: 
\begin{equation*} 
e^* \le \kappa\,\Big(1-\Dis{\PP_0}{\PP_1}\Big)
\end{equation*}
The above inequality holds with equality if $\kappa=\frac{1}{2}$. So, if distribution of the two classes are far away from each other in the total variation distance, then they can be distinguished from each other with high accuracy. 
\end{example}
\begin{proof}
The Bayes rate satisfies: 
\begin{align*}
e^* &= \sum_{x\in\X} \min\Big\{\Prob(X=x,\; Y=0),\; \Prob(X=x,\; Y=1)\Big\}\\
&=  \sum_{x\in\X} \min\Big\{\Prob(X=x|Y=0)\cdot \Prob(Y=0),\\
&\quad\quad\quad\quad\quad\quad \Prob(X=x|Y=1)\cdot\Prob(Y=1)\Big\}\\
&\le \kappa\; \sum_{x\in\X} \min\Big\{\Prob(X=x\;|\; Y=0),\; \Prob(X=x\;|\; Y=1)\Big\}\\
&=\kappa\,\Big(1-\Dis{\PP_0}{\PP_1}\Big)
\end{align*}
\end{proof}

\subsection{Mutual Affinity}
The second fundamental concept in this paper is \emph{mutual affinity}:
\begin{definition}[Mutual Affinity] 
The mutual affinity between two random variables $X_1$ and $X_2$ is defined by: 
\begin{align*}
\Inf{X_1}{X_2} &=\Dis{\Prob(X_1)\cdot \Prob(X_2)}{\Prob(X_1,\,X_2)} \\ 
&=\E_{X_1}\;\Dis{\Prob(X_2)}{\Prob(X_2\;|\;X_1)}\\
&=\E_{X_2}\;\Dis{\Prob(X_1)}{\Prob(X_1\;|\;X_2)}
\end{align*}
\end{definition}
Mutual affinity is quite analogous to  mutual information. In information theory, mutual information between two random variables $X_1$ and $X_2$ is the distance between the hypothesis that the two random variables are independent of each other vs. their true joint distribution, where distance is measured in the Kullback-Leibler divergence sense. In learning theory, mutual affinity is the distance between the same two hypotheses, where distance is now measured in the total variation sense. 
\begin{example}\label{MutAffinityExample_2}
Suppose $X$ and $Y$ are binary random variables with $\PP(Y=X)=1-\epsilon$ as depicted in Figure \ref{MutuAffinityFigure_2}. Then, $\Inf{X}{Y} = |\frac{1}{2}-\epsilon|$. Hence, if $\epsilon=\frac{1}{2}$, then $X$ and $Y$ are independent of each other and mutual affinity is identically zero.
\end{example}

\begin{figure}[H]
		\centering
		\includegraphics[width=5cm, height=4cm]{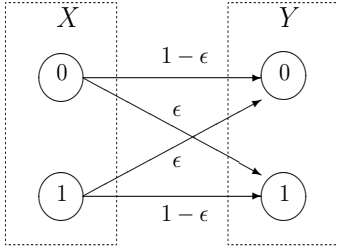}
		\caption{The mutual affinity between $X$ and $Y$ in this example is $|\frac{1}{2}-\epsilon|$.}
		\label{MutuAffinityFigure_2}
\end{figure}

\subsection{Effective Support Set Size}
Our third key  concept is the \emph{effective support set size}. For learning problems whose observations $Z$ are drawn from a probability mass function $\PP(z)$, size of the support set of $\PP(z)$ is important because it relates to how difficult the learning problem is. In other words,  if size of the support set of $\PP(z)$ is large, then large training sets are usually needed. For most problems of interest, however, size of the support set is infinite, hence such comparison is inappropriate. The correct measure of the spread of a probability distribution is given by its effective support set size. 

\begin{definition}[Effective Support Set Size]\label{resolutionDefinition} 
Given a  probability mass function $\PP(z)$ on an alphabet $\Z$, its effective support set size is defined by: 
\begin{equation} 
\RES\,[\PP(z)] = 1+\Big(\sum_{z\in\Z} \sqrt{\PP(z)\,(1-\PP(z))}\Big)^2
\end{equation}
\end{definition}
\begin{example}
At one extreme, let $\PP(z)$ be a \emph{uniform} probability mass function on a finite alphabet $|\mathcal{Z}|<\infty$. Then, $\RES\,[\PP(z)]=|\mathcal{Z}|$. In other words, effective support set size of a uniform distribution is equivalent to size of its true support set. At the other extreme, let $\PP(z)$ be a Kronecker delta distribution $\PP(z) = \delta_{z,z_0}$, whose entire probability mass is located at a single point $z_0$, then $\RES\,[\PP(z)] = 1$.   
\end{example}

\begin{example}\label{geometricResolutionExample}
A geometric distribution $\PP(z)=\alpha\,(1-\alpha)^{z-1}$ defined on the set of natural numbers $\mathcal{Z}=\{1,\,2,\,\ldots\}$ with probability of success $0< \alpha\le 1$ has a finite effective support set size with the following upper bound: 
\begin{equation*}
\RES[\PP(z)] \le 1+\frac{\alpha}{\big(1-\sqrt{1- \alpha}\big)^2}
\end{equation*}
Intuitively, we expect effective support set size to decrease when $ \alpha\to 1$. As the probability of success increases, higher values of the geometric random variable become less likely and the probability mass becomes more concentrated around the first few values. 
\end{example}
\begin{proof}
Effective support set size of a geometric distribution with probability of success $ \alpha>0$ satisfies: 
\begin{align*}
\RES[\PP(z)]&=1+\Big( \sum_{k=1}^\infty \sqrt{\Prob(k)}\,\sqrt{1-\Prob(k)}\Big)^2\\
&\le 1+\Big( \sum_{k=1}^\infty \sqrt{\Prob(k)}\Big)^2\\
&= 1+  \alpha\;\Big(\sum_{k=1}^\infty (1- \alpha)^{\frac{k-1}{2}}\,\Big)^2\\
&= 1+ \frac{ \alpha}{(1-\sqrt{1- \alpha})^2}
\end{align*}
\end{proof}

\subsection{Information and Events}
Our final key concept is that of \emph{information}. One of the most widely known results in information theory is that the amount of information delivered by an observation $Z$ is directly related to its uncertainty. In information theory, an observation $Z$ is said to deliver $\log{\frac{1}{\PP(Z)}}$ of information. Here, the use of $\log{\frac{1}{\PP(Z)}}$ is interpreted in terms of coding or \lq minimum description length'.  Such connection between information and uncertainty has found support in neuroscience \cite{knill2004bayesian,friston2010free}.

In the mathematical theory of learning presented in this paper, however, the amount of information delivered by an observation $Z$ is given by $1-\PP(Z)$. Informally speaking, the quantity of information delivered by an event, when measured in the context of learning, is the \lq\lq amount of change'' in our \lq\lq belief''upon knowing the event has occurred. Not surprisingly, information in the context of learning is also directly related to uncertainty.  

A precise definition of information is the following: 

\begin{definition}[Information]\label{inf_def}
Let $X$ and $Y$ be two random variables. The amount of information contained in an event $Y=y$ about  $X$ is given by:
\begin{equation}
\inff{X}{y} = \Dis{\PP(X)}{\PP(X\,|\,Y=y)} 
\end{equation}
In other words, the amount of information contained in an event $Y=y$ is the amount of change in our belief about the probability distribution of $X$ measured in the total variation sense.
\end{definition}

\begin{example}
For any random variable $X$ and any event $X=x$, we have: 
\begin{equation*}
\inff{X}{x} = \Dis{\PP(X)}{\PP(X\,|\,X=x)} = 1-\PP(x)
\end{equation*}
\end{example} 

It is perhaps worthwhile to note that the  measure of information in Definition \ref{inf_def} closely resembles the \lq\lq Bayesian Brain'', which is a popular model for how the brain might encode information coming from sensory systems. According to such model, the brain encodes its beliefs in the form of probabilities that are being continuously updated given new sensory information. For example, the depth of objects $D$ can have a prior distribution $\PP(D)$. Given a new retinal image $G$, the probability distribution of depth of an object changes to $\PP(D\,|\,G)$. The brain, then, acts upon the new posterior distribution \cite{knill2004bayesian,friston2010free,huang2008unified}. In this regard, a retinal image $G$ is only as useful as the impact it leaves on our prior beliefs, which is consistent with Definition \ref{inf_def}.

\begin{example}
The mutual affinity between two random variables $X$ and $Y$ is the expected amount of information one variable carries about the other:
\begin{equation}
\Inf{X}{Y}\;\; =\;\; \E_Y\;\inff{X}{Y}\;\;=\;\;\E_X\;\inff{Y}{X}
\end{equation} 
\end{example}

Finally, we note that for any two random variables $X$ and $Y$, the random variable $Y$ cannot carry more information about $X$ than it carries about itself, which is analogous to similar results in information theory.  More precisely, we have the following lemma: 
\begin{lemma} 
For any two random variables $X\in\X$ and $Y\in\Y$, and any event $Y=y$, we have: 
\begin{equation*}
\inff{X}{y} \le \inff{Y}{y}
\end{equation*} 
\end{lemma} 
\begin{proof}
Given an event $Y=y$, we have: 
\begin{align*}
\inff{X}{y} &= \Dis{\PP(X)}{\PP(X\,|\,Y=y)}\\ 
&=1-\sum_{x\in\X}\;\min\{\PP(X=x),\,\PP(X=x\,|\,Y=y)\}\\ 
&\le 1-\sum_{x\in\X}\;\min\{\PP(X=x,\,Y=y),\,\PP(X=x|Y=y)\}\\ 
&= 1-\sum_{x\in\X}\; \PP(X=x,\,Y=y)\; \min\{1,\,\frac{1}{\PP(Y=y)}\}\\ 
&= 1-\sum_{x\in\X}\;\PP(X=x,\,Y=y) \\ 
&= 1-\PP(Y=y)\\
&= \inff{Y}{y}
\end{align*}
The second inequality follows because $\PP(X)\ge \PP(X,\,Y)$.
\end{proof} 

\section{Main Results}
As stated earlier, our first goal is to be able to bound the difference between true and empirical risks of a given learning machine $\LL$. A formal  definition of $\LL$ was provided earlier in Section \ref{sec::notation}, which we interpreted by saying that $\LL$ receives a training set $S_m$ of $m$ i.i.d. examples $Z\sim\PP(z)$ and selects a hypothesis $H$ according to $\PP^{(m)}(h\,|\,S_m)$. Of course, $H$ can be selected deterministically, in which case $\PP^{(m)}(h\,|\,S_m)$ becomes a degenerate distribution. Once $H$ is selected, the entire learning process is concluded. 

Given an inferred hypothesis $H$, the same hypothesis can be used in multiple applications. For instance, if $H = \{c_j\}_{j=1,\ldots,k}\subseteq S_m$ is a set of $k$ prototypes that are selected out of $S_m$, then such prototypes can be used in clustering, and they can also be used regression or classification if the target concept $Y$ is one of the dimensions of $Z$. These different applications or uses of the \emph{same inferred hypothesis} give rise to \emph{different loss functions} $L_H(Z)$ that satisfy the Markov chain $S_m\to H\to L_H(Z)$. In regression, for instance, the loss function used might measure the mean-square error whereas a 0-1 misclassification loss might be used in classification. 

Nevertheless, the act of choosing a suitable loss function $L_H(Z)$ is \emph{outside} the learning process because the learning process is defined solely by how $H$ is being generated. Of course, the process by which $H$ is generated might be designed at the outset to optimize a specific loss function in mind, but this fact would already be encoded in the distribution $\PP^{(m)}(H\,|\,S_m)$. Henceforth, in order to bound the difference between true and  empirical risks of a given learning machine, we need a bound that holds simultaneously for \emph{all loss functions} that satisfy the Markov chain $S_m\to H\to L_H(Z)$, and we desire the bound to be as tight as possible. To achieve such objectives, we start with the following simple lemma. 
\begin{lemma}[Distance-Loss Lemma]\label{distanceLossLemma}
Suppose we have a random variable $Z\in\Z$ and two probability distributions $\Prob_1(z)$ and $\Prob_2(z)$ defined on $\Z$. Let $L(Z):\,\Z\to[0,\,1]$ be a loss function. Then: 
\begin{equation}
\Big| \E_{Z\sim\Prob_1(z)}L(Z)-\E_{Z\sim\Prob_2(z)}L(Z)\Big| \le \Dis{\Prob_1(z)}{\Prob_2(z)}
\end{equation}
In addition, there exists a loss function that achieves the bound. 
\end{lemma}
\begin{proof}
We have: 
\begin{align*}
\E_{Z\sim\Prob_1(z)}& L(Z)\,-\,\E_{Z\sim\Prob_2(z)}\,L(Z) \\
&= \sum_{z\in\Z}\int_{u=0}^{u=1} u\cdot \Prob(L(Z)=u|\,Z=z)\,\big[\Prob_1(z)-\Prob_2(z)\big] \mathrm{d}u\\
&\le \sum_{z\in\Z} \,\max\{\Prob_1(z)\,-\,\Prob_2(z),\;0\}\\
&= \Dis{\Prob_1(z)}{\Prob_2(z)}
\end{align*}
Now, consider the following upper bound:
\begin{equation*}
\E_{Z\sim\PP_1(z)}\,L(Z) \le \E_{Z\sim\PP_2(z)}\,L(Z) + \Dis{\PP_1(z)}{\PP_2(z)}
\end{equation*} 
To show that the bound above is tight, we  note that the inequality holds with equality for the loss function:
\begin{align*}
L^\star(Z)=\1\big\{\PP_1(Z)\ge \PP_2(Z)\big\}
\end{align*} 
Since $L^\star(Z)$ satisfies conditions of the theorem, namely that we have $L^\star(Z):\,\Z\to[0,\,1]$, the upper bound is tight. Tightness of the lower bound is derived similarly.
\end{proof} 

\begin{corollary}\label{theoremSandwidth0} 
Suppose in a learning machine $\LL$, the hypothesis $H$ is itself a valid bounded loss function $L(Z):\Z\to[0,\,1]$. That is, $\HH^{(m)}\subseteq \{f(z)|\;\forall z\in\Z: 0\le f(z)\le 1\}$ for all $m\ge 1$. Let $Z_{trn}$ be a random variable whose value is  drawn uniformly at random out of $S_m$ with replacement. Also, define $R(\LL)$ and $R_{emp}(\LL)$ by Eq \ref{true_risk_learn} and Eq \ref{emp_risk_eq} respectively, where $L_H(Z)$ is now the hypothesis $H$. Then:
\begin{equation}
\big|R(\LL)-R_{emp}(\LL)\big| \le \Inf{Z_{trn}}{H} 
\end{equation}
\end{corollary}
\begin{proof}
As always, we write $\PP(z)$ to denote the probability distribution of observations $Z$. First, we have by Eq \ref{true_risk_learn}:  
\begin{align}\label{cor1_RL_eq}
R(\LL) \,=\, \E_{S_m,H} \E_{Z\sim\PP(z)}\,L_H(Z) = \E_H\;\E_{Z\sim\PP(z)}\,L(Z)
\end{align} 
Now, we note that the value of the expression $L(Z)$ inside the expectation depends, in fact, on the value of  \emph{two random variables}. The first random variable is the choice of the loss function $H=L(Z)$, since this is selected by the learning machine according to $S_m$. The second random variable is the observation $Z$. However, by definition of \emph{true risk}, $Z$ is drawn from its original distribution $\PP(z)$ independently of $L(Z)$.

By contrast, the hypothesis $H=L(Z)$ and $Z_{trn}$  are not independent of each other since both clearly depend on the training set $S_m$. The probability of observing the pair $(H,\,Z_{trn})$ is $\PP(H)\cdot \PP(Z_{trn}\,|\,H)$, where by marginalization: 
\begin{align}\label{eq::train_distribution} 
\nonumber\PP(Z_{trn}|H) &= \sum_{s_m\in\Z^m}\,\PP(S_m=s_m|H)\cdot \PP(Z_{trn}|S_m=s_m,H)\\
&=  \sum_{s_m\in\Z^m}\PP(S_m=s_m|H)\cdot \PP(Z_{trn}|S_m=s_m)
\end{align} 
The last line follows because $Z_{trn}$ and $H$ are conditionally independent of each other given $S_m$. To simplify notation, we will use $\PP(z\,|\,H)$ for the conditional distribution of training examples given the hypothesis $H$. That is: 
\begin{equation}
\PP(z\,|\,H)\; \cong\; \PP(Z_{trn}=z\,|\,H),
\end{equation}
So, we have: 
\begin{align}\label{cor1_Remp_eq} 
\nonumber R_{emp}(\LL) &= \E_{H}\,\E_{S_m|H}\,\E_{Z_{trn}\sim S_m}\;L(Z_{trn})\\ 
&= \E_H\, \E_{Z_{trn}\sim\PP(z|H)}\;L(Z_{trn})
\end{align} 
The first line is our original definition of empirical risk given earlier in Eq \ref{emp_risk_eq}, while the second line follows by Eq \ref{eq::train_distribution}. Now, we employ Lemma \ref{distanceLossLemma} to deduce that: 
\begin{align*}
R(\LL) - R_{emp}(\LL) &= \E_{H}\,\E_{Z\sim\PP(z)}L(Z)-\E_H\E_{Z\sim\PP(z|H)}L(Z) \\ 
&= \E_{H}\,\Big[\E_{Z\sim\PP(z)}\,L(Z)\,-\,\E_{Z\sim\PP(z|H)}\,L(Z)\Big] \\ 
&\le  \E_H\,\Dis{\PP(z)}{\PP(z|H)}\\
&=\Inf{Z_{trn}}{H}
\end{align*}
In the first line, we substituted Eq \ref{cor1_RL_eq} and \ref{cor1_Remp_eq}.  In the third line, we employed Lemma \ref{distanceLossLemma}. The last line uses the fact that the marginal distribution of $Z_{trn}$ is $\PP(z)$ by assumption whereas its posterior distribution given $H$ is $\PP(z|H)$ as stated in Eq \ref{eq::train_distribution}. 

In a similar manner, we see by using Lemma \ref{distanceLossLemma} that the following lower bound holds: 
\begin{equation*} 
R(\LL) - R_{emp}(\LL)  \ge -\Inf{Z_{trn}}{H} 
\end{equation*}
Both bounds imply statement of the corollary.
\end{proof} 

\begin{example} 
At one extreme, if the loss function $L(Z):\,\Z\to[0,\,1]$ is chosen independently of the training set $S_m$, then $\Inf{Z_{trn}}{H}=0$ and we obtain $R(\LL)= R_{emp}(\LL)$. Hence, lack of learning is perfectly captured in this model. At the other extreme, if $\Z$ is a region in $\mathbb{R}^n$ and $\PP(z)$ is a bounded probability density function over $\Z$, then lazy learners such as in 1-NN classification achieve $\Inf{Z_{trn}}{H}=1$. In the latter case, the empirical risk $R_{emp}(\LL)$, which can always be made identically zero for some loss function $L_H(Z)$, might carry no information whatsoever about the true risk $R(\LL)$\footnote{The loss function $L_H(Z)$ that achieves the bound for lazy learners, whose hypothesis $H$ is itself the entire training set $S_m$,  is $L_H(Z) = \1\{Z\notin H\}$. If the observation space is a region in the plane $\mathbb{R}^n$ and $\PP(z)$ is a bounded density function, then $R(\LL)=1$ whereas $R_{emp}(\LL) = 0$. }. 
\end{example}

The previous corollary is restricted to the case in which the loss function $L(Z)$ is itself the learned hypothesis $H$. This happens, for example, in the classification setting if we seek a separating hyperplane $(w,b)$ for some normal vector $w\in\mathbb{R}^n$ and offset $b\in\mathbb{R}$, and use the loss function $L_{w,b}(x,y) = \1\{y(w^T\,x-b)\le 0\}$ to measure risk. Because $(w,\,b)\leftrightarrow L_{w,b}(\cdot)$ is a one-to-one mapping, the act of choosing a hypothesis $H=(w,\,b)$ is \emph{equivalent} to the act of choosing a loss function $H=L_{w,\,b}(Z)$. If the training set $S_m$ influences the choice of the loss function, which in turn is used to measure risk, then the statement of Corollary \ref{theoremSandwidth0} holds.

Next, we generalize the previous result to any arbitrary learning machine as shown in the following theorem. 
\begin{theorem}\label{theoremSandwidth1} 
Given a learning machine $\LL$ that receives a training set $S_m$ of $m$ i.i.d. examples $Z\sim\PP(z)$ and produces a hypothesis $H\in\mathcal{H}^{(m)}$, let $L_H(Z):\,\Z\to[0,\,1]$ be any loss function that satisfies the Markov chain $S_m\to H\to L_H(Z)$. Also, let $Z_{trn}$ be a random variable whose value is drawn uniformly at random out of $S_m$ with replacement. Then, the true risk $R(\LL)$ and empirical risk $R_{emp}(\LL)$  of  the learning machine, defined in Eq \ref{true_risk_learn} and Eq \ref{emp_risk_eq} respectively, are related by: 
\begin{equation}\label{theorem_bound_H_eq}
\big|R(\LL) - R_{emp}(\LL)\big| \le \Inf{Z_{trn}}{H} 
\end{equation}
In addition, this is the tightest possible bound. 
\end{theorem}
\begin{proof} 
We know from Corollary \ref{theoremSandwidth0} that for any loss function $L_H(Z)$ whose selection is influenced by the training set $S_m$, the following inequality holds:
\begin{equation*} 
\big|R(\LL) - R_{emp}(\LL)\big| \le \Inf{Z_{trn}}{L_H(Z)} 
\end{equation*} 
The quantity on the right-hand side is mutual affinity between  $Z_{trn}$ and the choice of the loss function $L_H(Z)$. Later in Section \ref{Sect::PartialOrder}, it will be shown using the data-processing inequality that $\Inf{A}{B}\ge \Inf{A}{C}$ whenever the Markov chain $A\to B \to C$ holds. Since $S_m\to H\to L_H(Z)$ implies $Z_{trn}\to H\to L_H(Z)$, we deduce the following bound:
\begin{equation*}
\Inf{Z_{trn}}{L_H(Z)} \;\le\; \Inf{Z_{trn}}{H},
\end{equation*} 
which holds whenever the Markov chain $S_m\to H\to L_H(Z)$ holds. This establishes the upper bound. To prove tightness, let us consider the following loss function:
\begin{equation}\label{LstarTheoremEq}
L_H^\star(z)=\1\Big\{ \Prob(Z_{trn}=z)\,\ge \Prob(Z_{trn}=z\,|\,H)\Big\}
\end{equation} 
In order to establish tightness of the bound in Eq \ref{theorem_bound_H_eq}, we first see that the loss function $L_H^\star(Z)$ satisfies both conditions of the theorem; namely that we have $\forall z\in\Z: 0\le L_H^\star(z)\le 1$ and $S_m\to H\to L_H^\star(Z)$. The second condition holds because any change to the original training set $S_m$ that does not alter the final hypothesis $H$ will not change the loss function $L_H^\star(Z)$. 

Now, it is immediate to observe that $L^\star_H(Z)$ achieves the bound. This is because if we measure true and empirical risks of $\LL$ using $L_H^\star(Z)$, we obtain: 
\begin{align*} 
R(\LL)&-R_{emp}(\LL)\\
&= \E_{H}\, \E_{Z\sim\PP(z)}\;L_H^\star(Z) - \E_{H}\, \E_{S_m|H}\E_{Z\sim S_m}\;L_H^\star(Z)\\ 
&= \E_H\;\Big[\E_{Z\sim\PP(z)}\;L_H^\star(Z) \;-\; \E_{Z\sim\PP(z|H)}\;L_H^\star(Z) \Big]\\ 
&= \E_H\;\Big[\E_{Z\sim\PP(z)}\;\1\{\PP(Z)\ge \PP(Z|H)\} \;\\
&\quad\quad\quad\quad-\; \E_{Z\sim\PP(z|H)}\;\1\{\PP(Z)\ge \PP(Z|H)\} \Big]\\ 
&= \E_H\;\Big[\sum_{z\in\Z}\; \big(\PP(z)-\PP(z|H)\big)\cdot\1\{\PP(z)\ge \PP(z|H)  \Big]\\ 
&= \E_H\;\Dis{\PP(z)}{\PP(z|H)}\\ 
&= \Inf{Z_{trn}}{H}
\end{align*}
Again, the second line follows from Eq \ref{eq::train_distribution} while the third line follows by construction of $L^\star_H(Z)$. If the inequality in the definition of $L^\star_H(Z)$ in Eq \ref{LstarTheoremEq} is reversed, we obtain: 
\begin{equation*} 
R(\LL)-R_{emp}(\LL) = - \Inf{Z_{trn}}{H}
\end{equation*} 
Hence, the bound is tight. 
\end{proof}

Of course, one does not usually know the original distribution $\PP(z)$ so computing $\Inf{Z_{trn}}{H}$ is not always possible. To resolve this issue, we introduce the notion of \emph{capacity} of learning machines. 
\begin{definition}[Learning Capacity]\label{capacityDefinition} 
Let $\LL$ be a learning machine  that receives a training set $S_m$ of $m$ i.i.d. examples $Z\sim\PP(z)$  and  produces a hypothesis $H\in\mathcal{H}^{(m)}$. Let $Z_{trn}$ be a random variable whose value is drawn uniformly at random out of $S_m$ with replacement. Then, \emph{capacity} of the learning machine $\LL$ is defined by: $C^{(m)}(\LL)=\sup_{\PP(z)} \Inf{Z_{trn}}{H}$, where the supremum is taken over all possible distributions of observations $Z$. 
\end{definition}
\begin{theorem}\label{theoremSandwidth2} 
Let $\LL$ be a learning machine  that receives a training set $S_m$ of $m$ i.i.d. examples $Z\sim\PP(z)$ for some unknown distribution $\PP(z)$ and  produces a hypothesis $H\in\mathcal{H}^{(m)}$. Also, let $L_H(Z):\,\Z\to[0,\,1]$ be any loss function  that satisfies the Markov chain $S_m\to H\to L_H(Z)$. In addition, let $Z_{trn}$ be a random variable whose value is drawn uniformly at random out of $S_m$ with replacement. Then, the true risk $R(\LL)$ and empirical risk $R_{emp}(\LL)$  of  the learning machine, defined in Eq \ref{true_risk_learn} and Eq \ref{emp_risk_eq} respectively, are related by: 
\begin{equation}
\big|R(\LL) - R_{emp}(\LL)\big| \le C^{(m)}(\LL),
\end{equation}
which holds for any distribution of observations $\PP(z)$. In addition, this is the tightest possible bound. 
\end{theorem}
\begin{proof} 
This follows by Definition \ref{capacityDefinition} and Theorem \ref{theoremSandwidth1}.
\end{proof}

 \begin{example} \label{toyProblemDetails}
Suppose observations $Z\in\{0,\,1\}$ are  Bernoulli trials with $\PP(Z=1)=\phi$, and that our learning machine $\LL$ summarizes the training set $S_m$ with the empirical average $H=\frac{1}{m}\,\sum_{i=1}^m\,Z_i$. Then, the capacity of this learning machine is asymptotically given by $C(\LL) \sim \frac{1}{\sqrt{2\,\pi\,m}}$. The proof is given in Appendix \ref{proof::example::bernoulli_problem}. Suppose, in addition, that a classifier predicts either 0 or 1 depending on which label occurs most often in the training set. In other words, we have the loss function: 
\begin{equation*} 
L_H(Z) = \1\{Z=1\}\cdot \1\{H<\frac{m}{2}\} \;+\;  \1\{Z=0\}\cdot \1\{H\ge\frac{m}{2}\} 
\end{equation*} 
Here, $L_H(Z)$ satisfies the conditions of Theorem \ref{theoremSandwidth2}. When $\phi=1/2$, Table \ref{toyProblemTrainTestError} shows simulations results for different values of $m$. As shown in table, the bound $|R(\LL)-R_{emp}(\LL)| \le C^{(m)}(\LL)$ holds with equality in this case \footnote{Precisely, the experiment run as follows. For each value of $m$, a total of $m$ training bits were generated. Each bit is either \lq 0' or \lq 1' with equal probability. The training error is the fraction of bits that are different from the majority. Expected test error is always $\frac{1}{2}$. This entire process was, then, repeated 1,000 times and averages are reported.}. 
\end{example}

\begin{table}[t]
\begin{center}
\begin{tabular}{|c|c|c|c|c|c|}
\hline
$m$ &10 & 25 & 50& 100 &  200\\
\hline
\;$R_{emp}(\LL)$\; &0.3780&0.4194&0.4426&0.4613&0.4712\\ 
\;$C^{(m)}(\LL)$\; &0.1230&0.0806 &0.0561 &0.0398 &0.0282\\  \hline
\end{tabular}
\end{center}
\caption{First row is empirical risk of the learning machine in Example \ref{toyProblemDetails}, which is estimated by averaging over 1,000 realizations of training error for randomly drawn training sets. The second row is theoretical learning capacity.}
\label{toyProblemTrainTestError}
\end{table}

\begin{example}\label{toyProblemExample2}
Suppose we decided to use the same learning machine $\LL$ in Example \ref{toyProblemDetails} but we would like to change our classifier. In this new classifier, we use empirical average on the training set and decide to either predict $y=1$ all the time or $y=0$ all the time but we choose to do so randomly according to the empirical distribution of the two labels \footnote{Precisely, we generate $m$ random bits where \lq 0' and \lq 1' are equally likely. In each training set, we compute the sum of observations $s$, and decide with probability $\frac{s}{m}$ to predict \lq 1' all the time.}. Let $L_H(Z)$ be the prediction error of this classifier. Clearly, $L_H(Z)$ satisfies the conditions of Theorem \ref{theoremSandwidth2}. Table \ref{toyProbRandClassifierTable} shows empirical risk and the predicted upper bound on true risk for different values of $m$. Here, both labels are assumed to be equally likely, which means that the true risk of $\LL$ is always  $\frac{1}{2}$. As shown in the table, the bound indeed holds, albeit the bound is slightly loose in this example. However, we knew earlier from  Example \ref{toyProblemDetails} that a loss function indeed exists, for which the upper bound holds with equality. 
\end{example}

\begin{table}
\begin{center}
\vspace{0.5cm}\begin{tabular}{|c|c|c|c|c|c|}
\hline
$m$ &10 & 25 & 50& 100 &  200\\
\hline
$R_{emp}(\LL)$ &0.4460&0.4811&0.4928&0.4947&0.4966\\
$R_{emp}(\LL)+C^{(m)}(\LL)$ &0.5690&0.5621&0.5488&0.5345&0.5248 \\\hline
\end{tabular}
\end{center}
\caption{In this problem, the empriical risk and theoretical upper bound on true risk, when using the loss function in Example \ref{toyProblemExample2}. The true risk is always $\frac{1}{2}$ for all values of $m$.} 
\label{toyProbRandClassifierTable}
\end{table}

\begin{example}[Randomized Learning Machine]\label{randomLearningExample}
Suppose our observation space is finite $|\Z|<\infty$. Given a training set $S_m$ of $m$ i.i.d. observations, let $N(z)$ denote the number of times $z\in\Z$ is observed in the training set. Suppose we have a learning machine $\LL$ whose final hypothesis $H$ is a \emph{single} value $H\in\Z$ that is selected randomly according to the empirical distribution $\Prob(H=z)=N(z)/m$. For example, if $\Z=\{1,\,2,\,3,\,4\}$ and the training set is $S_m=\{1,\,2,\,1,\,1,\,3,\,3\}$, then we have $\PP(H=1|S_m)=\frac{1}{2}$, $\PP(H=2|S_m)=\frac{1}{6}$, $\PP(H=3|S_m)=\frac{1}{3}$, and $\PP(H=4|S_m)=0$. The capacity of this learning machine is given by:
\begin{equation*}
C^{(m)}(\LL)= \frac{1}{m}\cdot \big(1-\frac{1}{|\Z|}\big)
\end{equation*}
\end{example} 
\begin{proof}
The proof is given in Appendix \ref{proofRandomLearningExample}.
\end{proof}

The objective of introducing Example \ref{randomLearningExample} is two-fold. First, when $|\Z|=2$, we see that this problem is quite similar to the previous Bernoulli problem in Example \ref{toyProblemDetails} except for the fact that $H$ is now a randomized summary statistic of the training set. The capacity of the two learning machines, however, are quite different. In the deterministic learning machine, we had $C^{(m)}(\LL) =O(1/\sqrt{m})$ whereas we have $C^{(m)}(\LL) =O(1/{m})$ in the randomized learning machine. Intuitively, we know that randomness should decrease capacity. 

Second, we note that our randomized learning machine in Example \ref{randomLearningExample} is related to the earlier randomized classifier in Example \ref{toyProblemExample2}.  This is because we can equivalently think of the latter classifier as a \emph{deterministic classifier} that receives a \emph{randomized hypothesis} $H$, instead of treating it as a randomized classifier that receives a deterministic hypothesis. With this new equivalent view, we note that the randomized learning machine in Example \ref{randomLearningExample} can be used to bound the difference between true risk and empirical risk in Table \ref{toyProbRandClassifierTable}. In fact, the bound now holds with equality. In other words, the difference between empirical and true risks in Table \ref{toyProbRandClassifierTable} is equal to the capacity of the learning machine in Example \ref{randomLearningExample}. Thus, \emph{for the same classifier, one might be able to find a better learning machine that yields tighter bounds}. Later, a more insightful interpretation of this fact will be provided when we show that learning machines  admit a partial order. 

Finally, we conclude this section with the following remark. Perhaps, one central goal of any  learning algorithm is to guarantee \emph{generalization}. That is, we would like to ensure that the empirical risk we estimate on a given training set is a valid approximation to the true risk we expect to obtain in the future. This is necessary because any learning algorithm has access to the empirical risk {only}, which can be minimized if the learning machine has sufficient capacity. The true risk, on the other hand, is inaccessible directly, and one can only minimize it by using a learning machine that generalizes well. 

\begin{definition}\label{GeneralizationDefinition}
A learning machine $\LL$ \emph{generalizes} if for all distributions of observations $\PP(z)$ and all  loss functions $L_H(Z):\,\Z\to[0,\,1]$ that satisfy the Markov chain $S_m\to H\to L_H(Z)$, we have $\lim_{m\to\infty}\,|R_{emp}(\LL)-R(\LL)| = 0$.
\end{definition}

\begin{definition}\label{finiteCapacityDefinition}
A learning machine $\LL$ has a finite capacity if $\lim_{m\to\infty}\,C^{(m)}(\LL)=0$. 
\end{definition}

It is important to distinguish learning machines with finite capacity from those with infinite capacity. This is partly due to the following result:

\begin{theorem}\label{generalizeIFFfiniteC}
A learning machine $\LL$ generalizes if and only if it has a finite capacity. 
\end{theorem}
\begin{proof} 
This follows from the fact that $\big|R(\LL)-R_{emp}(\LL)\big| \le C^{(m)}(\LL)$ is achievable for some distribution $\PP(z)$ and some loss function $L_H(Z)$ that satisfies the conditions of Definition \ref{GeneralizationDefinition}. Thus, in order for $\LL$ to generalize, we must have $\lim_{m\to\infty}\,C^{(m)}=0$. The converse also holds.
\end{proof}

Luckily,  most learning machines of interest have finite capacities. In fact, any learning machine with a countable observation space $\Z$ has a finite capacity. This follows from the Asymptotic Equipartition Property (AEP) in information theory. In simple terms, for any distribution of observations $\PP(z)$, the sequence $(Z_1,\,Z_2,\,\ldots, Z_m)$ becomes progressively closer to a sequence that is \emph{unique up to permutation}, and this happens as $m\to\infty$. This unique sequence is the one implied by the law of large numbers; i.e. for each $z\in\Z$, the fraction of times $z$ appears in the sequence is given by its probability $\Prob(z)$ \cite{cover2012elements}.  Because the learning machines we consider in this paper are always invariant to permutations of training examples, knowledge of the hypothesis $H$ typically yields little information about the training set as $m\to\infty$ because all sufficiently large training sets are nearly identical at the specified limit up to permutation. Such conclusion will be established more formally in Theorem \ref{superLearningMachineLemma_1}, in which we provide the rate of convergence. 

\section{Interpreting Learning Capacity} 
In this section, we provide several interpretations to learning capacity. To do this, we first note that any learning process is influenced by three key components:
\begin{enumerate} 
\item Observations $Z$ including their space $\Z$ and probability distribution $\PP(z)$. 
\item The inference process $\PP^{(m)}(H\,|\,S_m)$. 
\item The hypothesis space $\HH^{(m)}$. 
\end{enumerate} 

All three components influence the learning capacity. In particular, if we impose restrictions on any of these three components, we effectively limit the learning capacity. In this section, we explore such possibilities. 

First, we show that learning capacity is indeed a reasonable measure of quantifying how \lq\lq much'' has been learned out of the training set. In particular, we show using the data-processing inequality that the \lq\lq more'' we learn, the larger the learning capacity is. Second,  we show that having \emph{algorithmic stability} in the inference process is equivalent to having a finite learning capacity. Third, we show that if observations are restricted to a countable space $\Z$ with a finite effective support set size, then all learning machines have finite capacity. Finally, we explore connections between learning capacity and the hypothesis space  $\HH^{(m)}$. One, perhaps not quite surprising, result is that all learning machines have finite capacity if size of the hypothesis space is finite. The latter result, proved via information theoretic inequalities, is analogous to well-known results that have been established in the past using the Probably Approximately Correct (PAC) model. 

\subsection{Partial Ordering of Learning Machines}\label{Sect::PartialOrder}
Earlier in Example \ref{toyProblemDetails} and Example \ref{randomLearningExample}, we looked into two learning machines that were very similar to each other, yet with drastically different capacities. Let us briefly look into those two learning machines again. In both machines, observations are Bernoulli random variables $Z\in\{0,\,1\}$. The difference between the two learning machines lies in their method of computing their  hypothesis $H$: 
\begin{enumerate}
\item The first learning machine $\LL_{det}$ computes the empirical average of samples $H_{det} = \frac{1}{m}\,\sum_{i=1}^m\,Z_i$. 
\item The second learning machine $\LL_{rnd}$ also computes the empirical average of samples $H_{det}=\frac{1}{m}\,\sum_{i=1}^m\,Z_i$. However, its final hypothesis is $H_{rnd}\in\{0,\,1\}$, where $H_{rnd}$ is a Bernoulli random variable with probability of success $H_{det}$. 
\end{enumerate}
We noted that $C(\LL_{rnd})\le C(\LL_{det})$. Why should the latter result hold? In this section, we show that the latter inequality holds because we have the Markov chain $S_m\to H_{det}\to H_{rnd}$. In other words, because $H_{det}$ is necessarily \lq\lq more informative'' than $H_{rnd}$, the learning machine $\LL_{det}$ has a larger learning capacity than $\LL_{rnd}$. To establish this result, we begin with the following lemma:

\begin{lemma}\label{dataProcessLem1} 
Let $A\in\mathcal{A}$, $B\in\mathcal{B}$, and $C\in\mathcal{C}$ be three random variables. If $A\to B\to C$ forms a Markov chain, i.e.  $\Prob(C\,|\,A,\, B)=\Prob(C\,|\,B)$, then: 
\begin{equation*}
\Inf{A}{(B,\,C)} = \Inf{A}{B}
\end{equation*}
In other words, because $C$ is conditionally independent of $A$ given $B$, adding $C$ to $B$ does not create any additional affinity with  $A$. 
\end{lemma}
\begin{proof} 
We have:
\begin{align*}
\mathcal{I_P}(&A,\,(B,\,C)) \\
&= 1-\sum\,\min\big\{\Prob(A)\cdot \Prob(B,\,C),\;\; \Prob(A,\,B,\,C) \big\}\\
&= 1- \sum\, \min\big\{\Prob(A)\cdot \Prob(B)\cdot \Prob(C\,|\,B),\\
&\quad\quad\quad\quad\quad\quad\quad\quad \;\; \Prob(A, B)\cdot \Prob(C\,|\,A,\,B) \big\}\\
&= 1- \sum\,\Prob(C\,|\,B)\; \min\big\{\Prob(A)\cdot \Prob(B), \;\; \Prob(A,\, B) \big\}\\
&= 1- \sum\, \min\big\{\Prob(A)\cdot \Prob(B), \;\; \Prob(A,\, B) \big\}\\
&= \Dis{\Prob(A)\cdot \Prob(B)}{\Prob(A,\,B)}\\
&= \Inf{A}{B}\end{align*}
\end{proof} 

\begin{lemma}[Information Can't Hurt]\label{infoCantHurt} 
For any random variables  $A\in\mathcal{A}$, $B\in\mathcal{B}$, and $C\in\mathcal{C}$, we have: 
\begin{equation*}
\Inf{A}{(B,C)} \ge \Inf{A}{B}
\end{equation*}
In other words, adding $C$ to $B$ cannot reduce affinity with $A$.
\end{lemma} 
\begin{proof}
We have by definition:
\begin{align*}
&\mathcal{I_P}(A,\,(B,\,C)) \\
&= 1-\sum \min\big\{\Prob(A)\cdot \Prob(B,\,C),\;\;\Prob(A,\,B,\,C) \big\}\\
&= 1-\sum_{a\in\mathcal{A}}\,\Prob(A=a)\\
&\quad\quad\sum_{b\in\mathcal{B},\,c\in\mathcal{C}}\min\big\{\Prob(B=b,C=c),\,\Prob(B=b,C=c|A=a) \big\}\\
\end{align*}
However, the minimum of the sums is always larger than the sum of minimums. That is: 
\begin{equation*}
\min\big\{\sum_i \alpha_i, \; \sum_i\beta_i\big\} \ge \sum_i \min\{\alpha_i,\,\beta_i\}
\end{equation*}
Using marginalization $\PP(x)=\sum_y\,\PP(x,\,y)$ and the above inequality, we obtain:
\begin{align*} 
\mathcal{I_P}(&A,\,(B,\,C)) \\
&= 1-\sum_{a\in\mathcal{A}}\,\Prob(A=a)\\
&\quad\quad\quad\sum_{b\in\mathcal{B},\,c\in\mathcal{C}}\min\big\{\Prob(B=b,C=c),\Prob(B=b,C=c|A=a) \big\}\\
&\ge 1-\sum_{a\in\mathcal{A},b\in\mathcal{B}} \min\{\Prob(A=a)\,\Prob(B=b), \;\;\Prob(A=a,\,B=b)\} \\
&=\Inf{A}{B}
\end{align*}
\end{proof}

Lemma \ref{infoCantHurt} is the analog to the \lq\lq Information can't hurt'' inequality in information theory. In the context of learning, it simply states that adding more summary statistics about the training set cannot decrease mutual affinity. Thus, \emph{the \lq\lq more'' the summary statistics we use, the larger the learning capacity is}. Using both lemmas, we arrive at the important data-processing inequality. 

\begin{lemma}[Data Processing Inequality]\label{dataProcessIneqLemma}
Suppose we have the Markov chain: 
\begin{equation*}
Z_{trn}\to H_1\to H_2,
\end{equation*}
where $Z_{trn}\sim\PP(z)$. Then, the following inequality holds for any distributions of observations $\PP(z)$: 
\begin{equation*}
\Inf{Z_{trn}}{H_1} \ge \Inf{Z_{trn}}{H_2}
\end{equation*}
\end{lemma}
\begin{proof}
We have by Lemma \ref{dataProcessLem1} and Lemma \ref{infoCantHurt}: 
\begin{equation*}
\Inf{Z_{trn}}{H_1} = \Inf{Z_{trn}}{(H_1,\,H_2)}\ge \Inf{Z_{trn}}{H_2}
\end{equation*}
\end{proof}

The statement that manipulation hurts information has manifested in many contexts. In information theory, manipulation leads to loss of mutual information, and hence decreases the capacity of communication channels \cite{cover2012elements}. In Bayesian decision theory, manipulation leads to loss of information and hence reduces the optimal Bayes rate in classification \cite{devroye1996probabilistic}. In our context, manipulation leads to loss of information, and hence decreases the capacity of learning machines. Decreasing the capacity of learning machines, however, is not always disadvantageous since it can mitigate overfitting. 

As suggested earlier, the data processing inequality yields an insightful notion of \emph{partial ordering} of learning machines.

\begin{definition}[Subsets and Supersets]
Suppose we have two learning machines: $\LL_1$ that produces a hypothesis $H_1$ and $\LL_2$ that produces $H_2$ over the same observation space $\Z$. We say that $\LL_2$ is a subset of $\LL_1$, denoted $\LL_2\subseteq \LL_1$, if  behavior of the learning machine $\LL_2$ can be simulated completely by $\LL_1$. Mathematically, we have $\LL_2\subseteq\LL_1$ if the Markov chain $S_m\to H_1\to H_2$ holds. 
\end{definition} 
\begin{figure}[t]
		\centering
		\includegraphics[width=9cm,height=5cm]{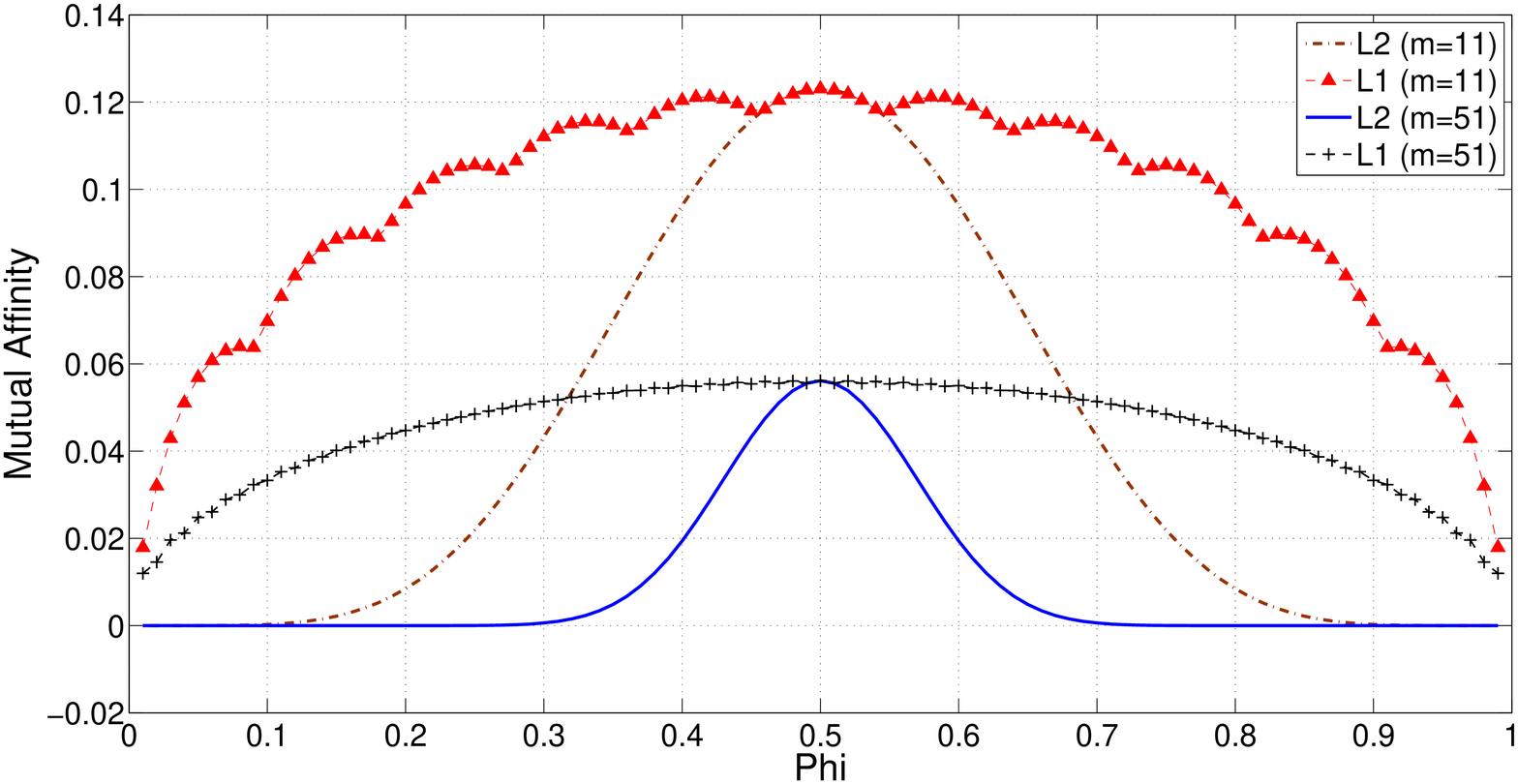}
		\caption{An illustration to a partial order of learning machines. Here, $\LL_2\subseteq \LL_1$, where the two machines are defined in Example \ref{partialOrderExample}. As shown in the figure, the inequality $\Inf{Z_{trn}}{H_2}\le \Inf{Z_{trn}}{H_1}$ always hold.}
		\label{DistanceThetaExmFig}
\end{figure}

\begin{theorem}\label{supersetCapacityTheorem}
If for two learning machines $\LL_1$ and $\LL_2$ we have $\LL_2\subseteq\LL_1$, then: 
\begin{equation*}
C^{(m)}(\LL_2)\;\le\; C^{(m)}(\LL_1),\quad\quad \text{for all } m\ge 1
\end{equation*}
\end{theorem}
\begin{proof}
By the data processing inequality (Lemma \ref{dataProcessIneqLemma}) and by definition of capacity (Definition \ref{capacityDefinition}). 
\end{proof} 

\begin{example}\label{partialOrderExample}
Returning again to our earlier example, where observations $Z\in\{0,\,1\}$ are Bernoulli trials with probability of success  $\phi$, suppose one learning machine $\LL_1$ computes the empirical average of observations $H_1=\frac{1}{m}\,\sum_{i=1}^m\,Z_i$. Also, suppose a second learning machine $\LL_2$ only reports the label that occurs most often in the training set. That is, $H_2=\1\{H_1\ge \frac{m}{2}\}$. Clearly, $\LL_2\subseteq \LL_1$. Figure \ref{DistanceThetaExmFig} shows mutual affinity $\Inf{Z_{trn}}{H_1}$ and $\Inf{Z_{trn}}{H_2}$ for $m=11$ and $m=51$ and different values of $\phi\in(0,1)$. As shown in the figure, the inequality $\Inf{Z_{trn}}{H_2}\le \Inf{Z_{trn}}{H_1}$ always hold. 
\end{example}
\begin{proof}
The proof is given in Appendix \ref{partialOrderExampleProof}. 
\end{proof} 

\subsection{Learning Capacity and Stability}
Algorithmic stability analysis was popularized in the last decade in learning theory. Because stability is a property of the learning machine itself rather than the hypothesis space, it is widely applicable to a broad class of machine learning algorithms. For this reason, stability has been proposed as a key condition for learnability and generalization \cite{bousquet2002stability,poggio2004general,shalev2010learnability}. In this section, we show that an appropriate notion of stability is indeed \emph{both necessary and sufficient} for generalization of learning machines. 

The notion of capacity or mutual affinity of a learning machine $\LL$ is intimately tied to the notion of stability. First, let us recall the inequality:
\begin{equation*}
\big|R(\LL)-R_{emp}(\LL)\big| \le  \Inf{Z_{trn}}{H}
\end{equation*}
For a given distribution $\PP(z)$, this is \emph{the tightest} possible bound because there exists some loss functions $L_H(Z):\Z\to[0,\,1]$ that satisfy the Markov chain $S_m\to H\to L_H(Z)$ and achieve the bound. To see how this is tied to the notion of stability, define: 
\begin{align}
\nonumber s^{(m)}(\LL) &= \E_{Z_{trn}}\; \Sim{\Prob^{(m)}(H)}{\Prob^{(m)}(H\,|\,Z_{trn})}\\
&=\sum_{z\in\Z}\,\Prob(Z_{trn}=z)\,\cdot\,\Sim{\Prob^{(m)}(H)}{\Prob^{(m)}(H\,|\,Z_{trn}=z)}
\end{align}
Here, $s^{(m)}(\LL)$ is a measure of how insensitive the learning machine $\LL$ is to a single training example, on average, when we have $m$ examples in the training set. Specifically, $\Prob^{(m)}(H\,|\,Z_{trn})$ is the probability distribution of the hypothesis $H$ given a single fixed training example $Z_{trn}$ and expectation is taken over all possible single training examples. If the learning machine is stable, then $\Prob^{(m)}(H)$ should be very close to $\Prob^{(m)}(H\,|\,Z_{trn})$ in distance and $s^{(m)}(\LL)\approx 1$. In this case, a single training example $Z_{trn}$ does not make a \lq\lq big'' difference to the distribution of the inferred hypothesis $H$. If we define stability of the learning machine $\LL$ using: 
\begin{equation}\label{stabilityEq}
\SSS^{(m)}(\LL) =   \inf_{\PP(z)}\;s^{(m)}(\LL),
\end{equation}
where the infimum is taken over all possible distributions of observations $Z$, then $\SSS^{(m)}(\LL)$ characterizes a distribution-free stability of $\LL$. However, we have by definition:
\begin{equation*}
C^{(m)}(\LL) = 1-\SSS^{(m)}(\LL)
\end{equation*} 
Thus, the capacity of a learning machine is inversely related to its algorithmic stability. Because in order for a learning machine to generalize (see Definition \ref{GeneralizationDefinition}) it is both necessary and sufficient that it has a finite capacity, we conclude that stability as defined in Eq \ref{stabilityEq} is also both necessary and sufficient. We emphasize again that such result holds for any learning machine including unsupervised learning algorithms. 

\begin{definition}\label{stabilityDefinition}
A learning machine $\LL$ is \emph{stable} if $ \lim_{m\to\infty} \SSS^{(m)}(\LL) = 1$
\end{definition} 
In other words, a learning machine is stable if the impact of a single observation becomes more and more negligible as size of the training set increases. 

\begin{theorem}\label{stab_theorem}
A learning machine $\LL$ generalizes if and only if it is algorithmically stable. 
\end{theorem}
\begin{proof} 
By Definition \ref{stabilityDefinition} and Theorem \ref{generalizeIFFfiniteC}. 
\end{proof} 

\begin{corollary}\label{stab_h_hbar_cor}
Suppose a learning machine $\LL$ is supplied with a training set $S_m$ that consists of $m$ i.i.d. training examples $Z\sim\PP(z)$, and let $H$ be the inferred hypothesis. Then, let  $S_m'$ be a new training set with $m$ i.i.d. observations drawn also from $\PP(z)$, and let $H'$ be the new hypothesis. Then: 
\begin{equation}
s^{(m)}(\LL)\ge \PP(H=H')
\end{equation}
\end{corollary}
\begin{proof}
Using the infinite product representation of the distance function $\Dis{\PP_1}{\PP_2}$ in Eq \ref{distUpperBound}, we know that:
\begin{equation}\label{loo_Eq_bound}
\Sim{\PP_1(z)}{\PP_2(z)} \;\ge\; \sum_{z\in\Z}\;\PP_1(z)\cdot \PP_2(z)
\end{equation} 
Now, we write by definition of stability and by Eq \ref{loo_Eq_bound}: 
\begin{align*} 
&\;s^{(m)}(\LL) = \E_{Z_{trn}} \Sim{\PP^{(m)}(H)}{\PP^{(m)}(H|Z_{trn})}\\ 
&\;\ge \E_{Z_{trn}} \sum_{h\in\HH^{(m)}} \PP^{(m)}(H=h)\cdot \PP^{(m)}(H=h|Z_{trn})\\ 
&\;= \E_{Z_{trn}}\sum_{h\in\HH^{(m)}} \E_{Z_{trn}'}\PP^{(m)}(H=h|Z_{trn}')\, \PP^{(m)}(H=h|Z_{trn})\\
&\;= \E_{Z_{trn}}\E_{Z_{trn}'}\sum_{h\in\HH^{(m)}} \PP^{(m)}(H=h|Z_{trn}')\, \PP^{(m)}(H=h|Z_{trn})\\ 
\end{align*} 
The last line states the following. First, we fix a single training example $Z_{trn}$ and draw all remaining $m-1$ training examples i.i.d. from $\PP(z)$ and let $H$ be the inferred hypothesis. After that, we perform a second trial, in which we fix a new training example $Z_{trn}'$ and let $H'$ be the new hypothesis. Then: 
\begin{align*} 
s^{(m)}(\LL) &= \E_{Z_{trn},\,Z_{trn}'}\; \PP(H=H'\;|\;Z_{trn},\,Z_{trn}')\\ 
&= \PP(H=H'),
\end{align*} 
where the second line follows by marginalization. 
\end{proof}

\begin{example}
If we return again to the Bernoulli problem, where we have binary observations $Z\in\{0,\,1\}$ that are drawn from a Bernoulli distribution with parameter $\Prob(Z=1)=\phi$. Let $\LL$ be the learning machine that produces the label the occurs most often in the training set. It was shown earlier in Example \ref{partialOrderExample} that capacity of this learning machine is given by (see Appendix \ref{partialOrderExampleProof}): 
\begin{equation*} 
C(\LL) \sim \frac{1}{\sqrt{2\,\pi\,m}}
\end{equation*} 
Capacity (maximum mutual affinity) occurs when $\phi=\frac{1}{2}$. For arbitrary values of $\phi$, mutual affinity is, in general, quite involved and is given in Appendix \ref{partialOrderExampleProof}. Instead of dealing with the exact expression of mutual affinity, we would like to draw qualitative results using stability analysis. Using Corollary \ref{stab_h_hbar_cor}, we note that if we draw two training sets $S_m$ and $S_m'$ independently, the probability we obtain different hypotheses is: 
\begin{align*}
\PP(H \neq H') &= 2\,\PP(H=0)\cdot \PP(H=1)\\
&=2\, \Big(\sum_{k=0}^{m/2}{m\choose k}\phi^k(1-\phi)^{m-k}\Big) \\
&\quad\quad \Big(\sum_{k=m/2+1}^{m}{m\choose k}\phi^k(1-\phi)^{m-k}\Big)\\ 
&\le 2\,\min\big\{\PP(H=0),\,\PP(H=1) \big\}
\end{align*}
Writing $\epsilon = |\frac{1}{2}-\phi|$ and using both Corollary \ref{stab_h_hbar_cor} and Hoeffding's inequality \cite{hoeffding1963}: 
\begin{align*}
\Inf{Z_{trn}}{H}&\le \PP(H\neq H')\\
&\le 2\,\exp\big\{-2\,m\,\epsilon^2 \big\}\\
&= 2\,\exp\big\{-2m\,\big|\frac{1}{2}-\phi\big|^2\big\}
\end{align*}
Clearly, this is a simple method of establishing that mutual affinity decreases exponentially fast when $\phi\neq \frac{1}{2}$, i.e. the two classes are not equally likely.
\end{example}

\subsection{Learning Capacity and Observations}
In the previous section, we looked into different interpretations of how the learning process influences the learning capacity. In this section, we look into observations $Z$ and the role of the effective support set size of $\PP(z)$. 

Earlier, it was stated that learning machines could be partially ordered, where $\LL_2\subseteq \LL_1$ implied that $C(\LL_2)\le C(\LL_1)$. In particular, if $\LL^\star$ is a lazy learner, then $C(\LL)\le C(\LL^\star)$ for all learning machines $\LL$. To reiterate, a lazy learner returns the training set itself as a hypothesis $H$. Next, we show that a lazy learner in a countable observation space actually has a finite capacity. 
\begin{theorem}[The Square-Root Law] \label{superLearningMachineLemma_1}
If observations $Z\in\Z$ are drawn i.i.d. from a probability distribution $\PP(z)$ with finite effective support set size, then the following asymptotic bound on capacity holds for any learning machine $\LL$ \footnote{Here, we have an additional term $o(1/\sqrt{m})$. However, such term is negligible and the bound becomes arbitrarily tight in ratio as $m\to\infty$.}:
\begin{equation}\label{SupersetLearningTheoremEq}
C^{(m)}(\LL) \le  \sqrt{\frac{\textnormal\RES\,[\PP(z)]-1}{2\,\pi\,m}} \le \sqrt{\frac{|\Z|-1}{2\,\pi m}}
\end{equation}
In addition, the lazy learner $\LL^\star$ achieves the bound.
\end{theorem} 
\begin{proof}
This can be proved by deducing capacity of the lazy learner $\LL^\star$ that is described earlier. The detailed proof is given in Appendix \ref{superLearningMachineProof_1}.
\end{proof}

Intuitively, Theorem \ref{superLearningMachineLemma_1} states that in order to have good generalization that holds for any possible learning machine, the average number of training examples per each possible observation must be sufficiently large. For \emph{multiclass} classification problems where $Z=(X,\,Y)\in\X\times\Y$, we have the following corollary: 
\begin{corollary} \label{universalCapacityClassificationCor}
Suppose observations consist of attributes plus labels, i.e. $Z=(X,\,Y)\in\X\times\Y$, where $|\X|\times |\Y|<\infty$, and our learning machine produces a hypothesis $H$, which is a classifier that predicts class label $Y$ given $X$. Let $L_H(Y,\,X)=\1\{Y\neq H(X)\}$ be the misclassification error. Also, let $\HH=\{h(x):\,\X\to\Y\}$ be the set of all possible hypotheses (classifiers). Then, the difference between empirical risk and true risk for any possible learning machine $\LL$ is asymptotically bounded by: 
\begin{align}\label{univCapCor_Eq1}
\Big|R(\LL)-R_{emp}(\LL)\Big| &\le \sqrt{\frac{|\X|\times |\Y|-1}{2\,\pi\,m}}\\
\label{univCapCor_Eq2}&= \sqrt{\frac{|\Y|\times  \log_{|\Y|}{|\HH|}-1}{2\,\pi\,m}}
\end{align}
\end{corollary}
\begin{proof}
We have $\Z=\X\times\Y$. Moreover: 
\begin{equation*}
|\HH| = |\Y|^{|\X|} \quad\quad \rightarrow \quad\quad  |\X| = \log_{|\Y|}{|\HH|}
\end{equation*}
Plugging these expressions into Theorem \ref{superLearningMachineLemma_1} yields the desired result. 
\end{proof}

Corollary \ref{universalCapacityClassificationCor} is quite similar to well-known results obtained using PAC model for binary classification problems. We will derive similar results later using information-theoretic bounds. It is perhaps worthwhile to point out that Theorem \ref{superLearningMachineLemma_1} can be interpreted as one additional formal justification to dimensionality reduction methods such as feature selection and Principal Component Analysis (PCA) because  reducing the effective support set size of observations helps improve generalization.

\subsection{Learning Capacity and Size of the Hypothesis Space}
Finally, we look into the role of the hypothesis space and how it influences learning capacity. So far, we have noted the apparent similarity between information theory and the learning theory proposed in this paper; in the sense that many quantities and results have analogs in both fields. There is, in addition, one concrete result that ties both fields together  \cite{cover2012elements}: 
\begin{theorem}[Pinsker's Inequaity]
For any two probability distributions $\PP_1(z)$ and $\PP_2(z)$, we have: 
\begin{equation*}
\Dis{\PP_1}{\PP_2} \le \sqrt{\frac{D(\PP_1\,||\,\PP_2)}{2}},
\end{equation*} 
where $D(\PP_1\,||\,\PP_2)$ is the Kullback-Leibler divergence measured in nats (i.e. using natural logarithms). 
\end{theorem}  

Many connections can be immediately deduced using Pinsker's inequality. For example, we have the following corollary: 
\begin{corollary}\label{cor_inf_ineq}
For any two random variables $X$ and $Y$, we have:
\begin{equation}
 \Inf{X}{Y} \le \sqrt{\frac{I(X,\,Y)}{2}},
\end{equation}
where $\Inf{X}{Y}$ is mutual affinity while $I(X,\,Y)$ is mutual information between $X$ and $Y$. 
\end{corollary}

Using Corollary \ref{cor_inf_ineq}, we obtain the following bound on capacity that holds for any learning machine $\LL$: 
\begin{theorem}\label{finite_summary_space}
Suppose we have a learning machine $\LL$ that receives a training set $S_m=\{Z_1,\ldots,\,Z_m\}$ and produces a hypothesis $H\in\HH^{(m)}$. Then the following bound  holds: \begin{equation*}
\Inf{Z_{trn}}{H} \le \sqrt{\frac{I(S_m,\,H)}{2\,m}}
\end{equation*} 
\end{theorem} 
\begin{proof}
We will write $\textbf{H}$ to denote the Shannon entropy. First, we note that: 
\begin{align*}
I(S_m,\,H) &= \textbf{H}(S_m) - \textbf{H}(S_m\,|\,H) \\ 
&= \sum_{i=1}^m\, \textbf{H}(Z_i)\, - \Big[\textbf{H}(Z_1|H) + \textbf{H}(Z_2|Z_1,H) + \cdots \Big] \\ 
&\ge  \sum_{i=1}^m\, [\textbf{H}(Z_i)-\textbf{H}(Z_i\,|\,H)]\\ 
&= m\,[\textbf{H}(Z_{trn})-\textbf{H}(Z_{trn}\,|\,H)] \\ 
&=m\,I(Z_{trn},\,H) 
\end{align*}
Here, the inequality follows because $\textbf{H}(A|B)\le \textbf{H}(A)$ for any random variables $A$ and $B$. The fourth line follows because we always assume that the learning process is invariant to permutation of the training set. Thus, we obtain: 
\begin{equation*}
I(Z_{trn},\,H) \le \frac{I(S_m,\,H)}{m}
\end{equation*} 
Combining this with Corollary \ref{cor_inf_ineq} yields the desired result. 
\end{proof} 

\begin{corollary}\label{finite_summary_space}
If  $\HH^{(m)}$ is a countable space and $H$ is the inferred hypothesis, then the following bound holds for all learning machines: 
\begin{equation*}
C^{(m)}(\LL) \le \sqrt{\frac{\textbf{\textup{H}}(H)}{2\,m}}\le \sqrt{\frac{\log{|\HH^{(m)}|}}{2\,m}},
\end{equation*} 
where $\textbf{\textup{H}}$ is the Shannon entropy. 
\end{corollary} 
\begin{proof}
Because for any random variables $A\in\mathcal{A}$ and $B\in\mathcal{B}$, we have $I(A,\,B)\le H(A)$, and $\textbf{H}(A)\le \log|\mathcal{A}|$. 
\end{proof} 

Corollary \ref{finite_summary_space} generalizes the well-known PAC result on the finite hypothesis space \cite{abu2012learning}. In fact, the bound in Corollary \ref{finite_summary_space} is tighter since $\log{|\HH^{(m)}|}$ is now replaced with entropy of the hypothesis $H$. From Corollary \ref{finite_summary_space}, we can effortlessly deduce the following bound: 
\begin{corollary} 
If we have a finite observation space $|\Z|<\infty$, then the following  bound on capacity holds for any learning machine $\LL$:
\begin{equation}
C^{(m)}(\LL) \le \sqrt{\frac{|\Z|\cdot \log{(1+m)}}{2\,m}},
\end{equation}
which is consistent with Theorem \ref{superLearningMachineLemma_1}. 
\end{corollary} 
\begin{proof}
In the language of information theory, using the \emph{method of types} to be more specific,  the discrete lazy learner $\LL^\star$ produces  the hypothesis $H = T[S_m]$, which is the \emph{type} of the training set $S_m$. Here, the type of a training set is its empirical probability mass function. However, it is well-known that the number of possible types given $m$ training examples is always bounded by $\HH^{(m)} \le (1+m)^{|\Z|}$ \cite{cover2012elements}.  Combining this with Corollary \ref{finite_summary_space} yields the desired result. 
\end{proof}

The reason behind introducing last corollary is to illustrate one scenario where information theory simplifies analysis in learning theory. Originally, Theorem \ref{superLearningMachineLemma_1} provided us with the tightest possible bound that is achievable by the discrete lazy learner $\LL^\star$. However, its proof is rather involved and is combinatorial in nature. By contrast, the proof of  last corollary is quite simple, albeit at a cost of obtaining a slightly looser bound. 

\section{Conclusions} 
This paper proposes a new mathematical theory of learning. Unlike earlier approaches, the theory presented here does not treat learning as a problem of convergence of random variables to their means and does not rely on concentration inequalities.

The theory enjoys many advantages. First, it ties the mathematical notion of learning to the mathematical notion of information. For example, mutual affinity in learning theory is quite similar to mutual information, capacity of learning machines is analogous to capacity of communication channels, and the asymptotic equipartition property (AEP) as well as the data-processing inequality both play key roles in the two theories. Second, the bounds obtained through this theory are the tightest possible bounds. Third, the theory follows Vapnik's General Setting of Learning, which is a unified approach towards analyzing many learning tasks including supervised and unsupervised learning algorithms. 

Perhaps, the best statement to conclude this paper with is to summarize the different interpretations of learning capacity  that have been deduced so far. We have the following results: 
\begin{enumerate} 
\item The capacity of a learning machine is a measure of the maximum difference between empirical and true risks  $\Big|R(\LL)-R_{emp}(\LL)\Big|$. Because bounds are tight, a learning machine generalizes if and only if it has a finite capacity.
\item The capacity of a learning machine is a measure that quantifies how much is expected to be learned out of the training set. Hence, adding more summary statistics increases capacity of the learning machine. If one learning machine $\LL_1$ can be completely simulated by a second learning machine $\LL_2$, i.e. $\LL_2$ is necessarily more informative than $\LL_1$, then we have $C(\LL_2)\ge C(\LL_1)$. 
\item The capacity of a learning machine is a measure of its algorithmic instability. Learning machines whose inferred hypothesis $H$ is  heavily perturbed by a change in a single training example have a higher capacity. Moreover, a learning machine generalizes if and only if it is stable. 
\item The capacity of a learning machine is limited by the effective support set size of observations. If observations have a finite effective support set size, then sufficiently large training sets will effectively exhaust the space of possible observations and all learning machines generalize as a result. 
\item A learning machine is limited by the size of its hypothesis space $\HH^{(m)}$. If the hypothesis space $\HH^{(m)}$ is finite in size, then all  learning machines have finite capacity that grows only logarithmically with $|\HH^{(m)}|$. 
\end{enumerate} 

\bibliography{MathLearning}

\begin{thebibliography}{10}
\providecommand{\url}[1]{#1}
\csname url@samestyle\endcsname
\providecommand{\newblock}{\relax}
\providecommand{\bibinfo}[2]{#2}
\providecommand{\BIBentrySTDinterwordspacing}{\spaceskip=0pt\relax}
\providecommand{\BIBentryALTinterwordstretchfactor}{4}
\providecommand{\BIBentryALTinterwordspacing}{\spaceskip=\fontdimen2\font plus
\BIBentryALTinterwordstretchfactor\fontdimen3\font minus
  \fontdimen4\font\relax}
\providecommand{\BIBforeignlanguage}[2]{{%
\expandafter\ifx\csname l@#1\endcsname\relax
\typeout{** WARNING: IEEEtran.bst: No hyphenation pattern has been}%
\typeout{** loaded for the language `#1'. Using the pattern for}%
\typeout{** the default language instead.}%
\else
\language=\csname l@#1\endcsname
\fi
#2}}
\providecommand{\BIBdecl}{\relax}
\BIBdecl

\bibitem{vapnik1999overview}
V.~N. Vapnik, ``An overview of statistical learning theory,'' \emph{Neural
  Networks, {IEEE} Transactions on}, vol.~10, no.~5, September 1999.

\bibitem{blumer1989learnability}
A.~Blumer, A.~Ehrenfeucht, D.~Haussler, and M.~K. Warmuth, ``Learnability and
  the {V}apnik-{C}hervonenkis dimension,'' \emph{Journal of the {ACM} (JACM)},
  vol.~36, no.~4, pp. 929--965, 1989.

\bibitem{talagrand1996majorizing}
M.~Talagrand, ``Majorizing measures: the generic chaining,'' \emph{The Annals
  of Probability}, vol.~24, no.~3, pp. 1049--1103, 1996.

\bibitem{mcallester1999some}
D.~A. McAllester, ``Some {PAC}-{B}ayesian theorems,'' \emph{Machine Learning},
  vol.~37, pp. 355--363, 1999.

\bibitem{mcallester2003pac}
------, ``{PAC-B}ayesian stochastic model selection,'' \emph{Machine Learning},
  vol.~51, pp. 5--21, 2003.

\bibitem{bousquet2002stability}
O.~Bousquet and A.~Elisseeff, ``Stability and generalization,'' \emph{The
  Journal of Machine Learning Research {(JMLR)}}, vol.~2, pp. 499--526, 2002.

\bibitem{bartlett2002rademacher}
P.~L. Bartlett and S.~Mendelson, ``Rademacher and gaussian complexities: Risk
  bounds and structural results,'' \emph{The Journal of Machine Learning
  Research {(JMLR)}}, vol.~3, pp. 463--482, 2002.

\bibitem{audibert2007combining}
J.-Y. Audibert and O.~Bousquet, ``Combining {PAC-B}ayesian and generic chaining
  bounds,'' \emph{The Journal of Machine Learning Research {(JMLR)}}, vol.~8,
  pp. 863--889, 2007.

\bibitem{xu2012robustness}
H.~Xu and S.~Mannor, ``Robustness and generalization,'' \emph{Machine
  learning}, vol.~86, no.~3, pp. 391--423, 2012.

\bibitem{pacsemisupervised2005}
M.-F. Balcan and A.~Blum, ``A {PAC}-style model for learning from labeled and
  unlabeled data,'' \emph{Learning Theory}, vol. 3559, pp. 111--126, 2005.

\bibitem{cover2012elements}
T.~M. Cover and J.~A. Thomas, \emph{Elements of information theory}.\hskip 1em
  plus 0.5em minus 0.4em\relax Wiley \& Sons, 1991.

\bibitem{reiser1999confidence}
B.~Reiser and D.~Faraggi, ``Confidence intervals for the overlapping
  coefficient: the normal equal variance case,'' \emph{The Statistician},
  vol.~48, no.~3, pp. 413--418, 1999.

\bibitem{knill2004bayesian}
D.~C. Knill and A.~Pouget, ``The bayesian brain: the role of uncertainty in
  neural coding and computation,'' \emph{TRENDS in Neurosciences}, vol.~27,
  no.~12, pp. 712--719, 2004.

\bibitem{friston2010free}
K.~Friston, ``The free-energy principle: a unified brain theory?'' \emph{Nature
  Reviews Neuroscience}, vol.~11, no.~2, pp. 127--138, 2010.

\bibitem{huang2008unified}
G.~Huang, ``Is this a unified theory of the brain,'' \emph{New Scientist}, vol.
  2658, pp. 30--33, 2008.

\bibitem{devroye1996probabilistic}
L.~Devroye, L.~Gy{\"o}rfi, and G.~Lugosi, \emph{A probabilistic theory of
  pattern recognition}.\hskip 1em plus 0.5em minus 0.4em\relax Springer, 1996.

\bibitem{poggio2004general}
T.~Poggio, R.~Rifkin, S.~Mukherjee, and P.~Niyogi, ``General conditions for
  predictivity in learning theory,'' \emph{Nature}, vol. 428, pp. 419--422,
  2004.

\bibitem{shalev2010learnability}
S.~Shalev-Shwartz, O.~Shamir, N.~Srebro, and K.~Sridharan, ``Learnability,
  stability and uniform convergence,'' \emph{The Journal of Machine Learning
  Research {(JMLR)}}, vol.~11, pp. 2635--2670, 2010.

\bibitem{hoeffding1963}
W.~Hoeffding, ``Probability inequalities for sums of bounded random
  variables,'' \emph{Journal of the American Statistical Association}, vol.~58,
  no. 301, pp. 13--30, 1963.

\bibitem{abu2012learning}
Y.~S. Abu-Mostafa, M.~Magdon-Ismail, and H.-T. Lin, \emph{Learning from data},
  2012.

\bibitem{WeissteinBinomialDistribution}
E.~W. Weisstein, ``Binomial distribution.'' [Online] From MathWorld--A Wolfram
  Web Resource. \url{http://mathworld.wolfram.com/BinomialDistribution.html},
  2013, accessed: 2013-06-30.

\end{thebibliography}

\appendices
\section{Proof of Example \ref{toyProblemDetails}.} \label{proof::example::bernoulli_problem}
First, $H$ has a binomial distribution: 
\begin{equation*}
\PP(H=k/m) = {{m}\choose{k}}\,\phi^{k}\;(1-\phi)^{m-k}
\end{equation*}
We use the identity: 
\begin{align*}
\Inf{Z_{trn}}{H} &= 1- \sum_{h\in\HH^{(m)}}\,\Prob(H=h)\\ 
&\quad\quad\quad\sum_{z\in\Z}\min\big\{\Prob(Z_{trn}=z),\Prob(Z_{trn}=z|H=h \big\}\\
&= \sum_{h\in\HH^{(m)}}\Prob(H=h)\cdot\Dis{\Prob(Z_{trn})}{\Prob(Z_{trn}|H=h)}
\end{align*} 
However, $\Prob(Z_{trn})$ is a Bernoulli distribution with probability of success $\phi$ while $\Prob(Z_{trn}\,|\,H=h)$ is Bernoulli with probability of success $h$. Knowing that the distance between two Bernoulli distributions is given by $|\phi-h|$, we obtain:
\begin{equation}\label{mutual_afffinity_bernoulli_example}
\Inf{Z_{trn}}{H} = \sum_{k=0}^m {{m}\choose{k}}\,\phi^{k}\;(1-\phi)^{m-k}\;\Big|\phi-\frac{k}{m}\Big|
\end{equation}
This is the \emph{mean deviation} of the binomial distribution. Assuming $\phi\,m$ is an integer, then the above expression is given by \cite{WeissteinBinomialDistribution}:
\begin{equation}\label{MDEq}
MD = \frac{2}{m}\,(1-\phi)^{(1-\phi)\,m}\;\phi^{1+m\,\phi}\,(1+m\,\phi)\,{m\choose {m\,\phi+1}}
\end{equation}
The maximum mutual affinity is achieved when $\phi=\frac{1}{2}$. This gives us: 
\begin{align*}
C^{(m)}(\LL) &= \frac{1}{2^{m+1}} \frac{m!}{\big((m/2)!\big)^2}\\
& \sim \frac{1}{\sqrt{2\,\pi\,m}},
\end{align*}
where in the last step we used Stirling's approximation. 

\section{Proof of Example \ref{randomLearningExample}.}\label{proofRandomLearningExample} 
We will take the extreme case where all observations in $\Z$ are equally likely. Intuitively, this corresponds to the most difficult distribution to learn. 
Then, we have by symmetry: 
\begin{equation*}
\Prob(H=z) = \frac{1}{|\Z|}
\end{equation*}
Since $\PP(H=z) = \PP(Z_{trn}=z)$, we have by Bayes rule: 
\begin{align*}
\Prob(Z_{trn}\,|\,H) = \Prob(H\,|\,Z_{trn})
\end{align*}
However, given a single random draw of a training example $Z_{trn}$ out of $S_m$, the probability of eventually selecting a label $H=z$ depends on two cases: 
\begin{equation*}
\Prob(H=z'\,|\,Z_{trn}=z) = \begin{cases} Q & \text{if } z'=z \\ R & \text{if } z'\neq z \end{cases}
\end{equation*}
Of course, we have $Q+(|\Z|-1)\,R=1$. To find $Q$, we use the definition of $\LL$: 
\begin{equation*}
Q=\frac{1}{m}\cdot \big(1+\frac{m-1}{|\Z|}\big) = \frac{1}{m}+\frac{1}{|\Z|}\cdot\frac{m-1}{m}
\end{equation*}
Note that we used the fact that the learning machine is randomized in deriving the above expression for $Q$. So, to satisfy $Q+(|\Z|-1)\,R=1$, we have: 
\begin{equation*}
R = \frac{1}{|\Z|}\cdot \frac{m-1}{m}
\end{equation*}
Now, we are ready to find the desired expression: 
\begin{align*}
\Prob(Z_{trn}=z\,|\,H=z') &= \PP(H=z'\,|\,Z_{trn}=z)\\
&= \1\{z=z'\}\cdot\, Q + \1\{z\neq z'\}\cdot\,R  \\
&=\frac{\1\{z=z'\}}{m}\;+ \frac{m-1}{m\,|\Z|} \\
&= \Big( \1\{z=z'\} -\frac{1}{|\Z|}\Big)\cdot \frac{1}{m} +\frac{1}{|\Z|}
\end{align*}
So, the joint distribution of $H$ and $Z_{trn}$ is:
\begin{align*}
\PP(H=z',Z_{trn}=z) &= \PP(H=z')\cdot \Prob(Z_{trn}=z\,|\,H=z') \\ 
&= \frac{1}{|\Z|^2}\, +\, \Big( \1\{z=z'\}-\frac{1}{|\Z|}\Big)\, \frac{1}{m\,|\Z|}\\
&= \PP(H=z')\,\PP(Z_{trn}=z)\,\\
&\quad\quad\quad+ \Big( \1\{z=z'\}-\frac{1}{|\Z|}\Big)\, \frac{1}{m\,|\Z|}
\end{align*}
Since $|\Z|>1$, we see from the last expression that: 
\begin{equation*}
\PP(H,Z_{trn})>\PP(H)\cdot\PP(Z_{trn})\quad \leftrightarrow\quad H=Z_{trn}
\end{equation*} 
Hence, mutual affinity is given by: 
\begin{align*}
\Inf{Z_{trn}}{H}&=\Dis{\PP(Z_{trn})\cdot\PP(H)}{\PP(Z_{trn},\,H)} \\ 
&= 1 - \sum_{z,z'\in\Z} \min\big\{\PP(H=z',\,Z_{trn}=z),\\ 
&\quad\quad\quad\quad\quad\quad\quad\quad\;\;\PP(H=z')\cdot\PP(Z_{trn}=z) \big\}\\
&= 1 - \sum_{z,z'\in\Z} \Big[\PP(H=z')\cdot\PP(Z_{trn}=z)\\ 
&\quad\quad+\;\min\Big\{0,\,\big(\1\{z=z'\}-\frac{1}{|\Z|}\big)\cdot \frac{1}{m\,|\Z|}\Big\} \Big]\\
&=- \sum_{z,z'\in\Z} \min\Big\{0,\big(\1\{z=z'\}-\frac{1}{|\Z|}\big)\, \frac{1}{m\,|\Z|} \Big\} \\ 
&= \frac{|\Z|-1}{m\,|\Z|} = \frac{1}{m}\cdot\big(1-\frac{1}{|\Z|} \big)
\end{align*}

\section{Proof of Example \ref{partialOrderExample}}\label{partialOrderExampleProof}
For the first learning machine, we have already shown in Example \ref{toyProblemDetails} that: 
\begin{equation*}
\Inf{Z_{trn}}{H} = \sum_{k=0}^m {{m}\choose{k}}\,\phi^{k}\;(1-\phi)^{m-k}\;\Big|\phi-\frac{k}{m}\Big| 
\end{equation*}
For the second learning machine $\LL_2$, we will assume that $m$ is odd. Then, the probability that $H_2=0$ is given by: 
\begin{equation*}
\PP(H_2=0) = \sum_{k=0}^{(m-1)/2}\;{{m}\choose{k}}\,\phi^{k}\;(1-\phi)^{m-k}
\end{equation*}
Knowing $H_2$, the marginal distribution of training examples is given by: 
\begin{equation*}
\PP(Z_{trn}=1|H_2=0) =\frac{\sum_{k=0}^{(m-1)/2}\,{{m}\choose{k}}\frac{k}{m}\,\phi^{k}\,(1-\phi)^{m-k}}{\PP(H_2=0)}
\end{equation*}
On the other hand: 
\begin{align*}
\PP(Z_{trn}=1\,|\,H_2=1) &= \frac{\sum_{k=(m-1)/2+1}^m {{m}\choose{k}}\,\frac{k}{m}\,\phi^{k}\,(1-\phi)^{m-k}}{\PP(H_2=1)}
\end{align*}
The mutual affinity is given by: 
\begin{align*}
\Inf{Z_{trn}}{H_2}&= \E_{H_2}\,\big|\phi-\Prob(Z_{trn}=1\,|\,H_2)\big|\\
&= \PP(H_2=0)\cdot|\phi-\PP(Z_{trn}=1|H_2=0)|\\
&\quad\quad+ \;\PP(H_2=1)\cdot|\phi-\PP(Z_{trn}=1|H_2=1)|\\
&= \Big|\phi-\sum_{k=0}^{(m-1)/2}\,{{m}\choose{k}}\,\phi^{k}\,(1-\phi)^{m-k}\,\Big(\phi-\frac{k}{m}\Big) \Big| \\
&\quad+ \Big|\phi-\sum_{k=\frac{m+1}{2}}^m {{m}\choose{k}}\,\phi^{k}\,(1-\phi)^{m-k}\,\Big(\phi-\frac{k}{m}\Big) \Big|\\
\end{align*}
This is the expression used in Figure \ref{DistanceThetaExmFig}. 

\section{Proof of Theorem \ref{superLearningMachineLemma_1}}\label{superLearningMachineProof_1}

First, suppose observations have a finite support $|\Z|<\infty$. To simplify notation, we will assume without loss of generality that $\Z=\{1,\,2,\,\ldots, \,|\Z|\}$. For a lazy learner $\LL^\star$, we note that its hypothesis $H$ is itself the entire training set $S_m$ up to permutation. Let $m_i$ be equal to the number of times $i\in\Z$ was observed in the training set, and let $p_i=\PP(Z=i)$. Then, we have: 
\begin{equation*}
\PP(H)= \PP(S_m) = {m\choose m_1,\,m_2,\,\ldots,\,m_{|\Z|}}\,p_1^{m_1}\,p_2^{m_2}\cdots p_{|\Z|}^{m_{|\Z|}}
\end{equation*}
Here, ${\cdot\choose\cdot}$ is the multinomial coefficient. For now, assume that maximum affinity is attained at the uniform distribution (this will be established formally using the effective support set bound proved later). Letting $p_i=\frac{1}{|\Z|}$, we obtain: 
\begin{equation*}
\PP(H)= \frac{1}{|\Z|^m}{m\choose m_1,\,m_2,\,\ldots,\,m_{|\Z|}}
\end{equation*}
Using the identity $\Dis{p}{q}=\frac{1}{2}||p-q||_1$, we obtain: 
\begin{align*}
&C(\LL^\star) =  \E_H\,\Dis{\PP(z)}{\PP(z|H)}\\
&=\frac{1}{2}\,\frac{1}{|\Z|^m} \sum_{k=1}^{|\Z|}\sum_{m_1+\ldots+m_{|\Z|}=m}{m\choose m_1,m_2,\ldots,m_{|\Z|}}\,\Big|\frac{m_k}{m}-\frac{1}{|\Z|}\Big| \\
&=\frac{1}{2}\,\frac{1}{|\Z|^{m-1}}\; \sum_{m_1+\ldots+m_{|\Z|}=m}\,{m\choose m_1,m_2,\ldots,m_{|\Z|}}\Big|\frac{m_1}{m}-\frac{1}{|\Z|}\Big|
\end{align*}
The second line follows by symmetry. We can simplify further:
\begin{align*}
C(\LL^\star) &=\frac{1}{2}\,\frac{1}{|\Z|^{m-1}}\; \sum_{k=0}^m\,{m\choose k}\,\big(|\Z|-1)^{m-k}\;\Big|\frac{k}{m}-\frac{1}{|\Z|}\Big|\\
&= \frac{|\Z|}{2\,m}\; \sum_{k=0}^m\,{m\choose k}\,\Big(\frac{1}{|\Z|}\Big)^k\;\Big(1-\frac{1}{|\Z|}\Big)^{m-k}\;\Big|k-\frac{m}{|\Z|}\Big|
\end{align*}

The last manipulation is intended to place the expression in a binomial distribution form. The expression is identical to was derived earlier in Example \ref{toyProblemDetails}, where we had $|\Z|=2$. Again, the quantity inside the summation is the \emph{mean deviation}. Using Eq \ref{MDEq} and simplifying yields:
\begin{align*}
C(\LL^\star) \sim \sqrt{\frac{|\Z|-1}{2\,\pi m}}
\end{align*}
In addition, the asymptotic relation is tight, in the sense that the ratio of the two terms goes to unity as $m\to\infty$. 

In the general case where $\PP(z)$ has a finite \emph{effective} support set size, the proof is quite similar to the above approach. Here, we note that for a given alphabet $\Z=(1,\,2,\,\ldots)$ and a fixed distribution $p_i=\PP(Z=i)$, we have: 
\begin{align*}
&\Inf{Z_{trn}}{H} = \frac{1}{2} \sum_{k}\\
&\quad\sum_{m_1+m_2+\ldots=m}{m\choose m_1,\,m_2,\,\ldots}\,p_1^{m_1}\,p_2^{m_2}\cdots\;\Big|\frac{m_k}{m}-p_k\Big|
\end{align*}
For the inner summation, we write: 
\begin{align*}
&\sum_{m_1+m_2+\ldots=m}\,{m\choose m_1,\,m_2,\,\ldots}\,p_1^{m_1}\,p_2^{m_2}\cdots\;\Big|\frac{m_k}{m}-p_k\Big|\\
&=\sum_{s=0}^m\,{m\choose s}\,p_k^s\,\Big|\frac{m_k}{m}-p_k\Big| \times \\
&\sum_{m_1+\ldots+m_{k-1}+m_{k+1}+\ldots=m-s}\,{m-s\choose m_1,\ldots,m_{k-1},\,m_{k+1},\ldots}\\
&\quad\quad\quad\quad\quad\quad \quad\quad\quad \quad\quad\quad \quad\quad p_1^{m_1}\cdots p_{k-1}^{m_{k-1}}\,p_{k+1}^{m_{k+1}}\cdots\\
&=\sum_{s=0}^m\,{m\choose s}\,p_k^s\,(1-p_k)^{m-s}\Big|\frac{m_k}{m}-p_k\Big|
\end{align*}
In the last step, we used the multinomial series. Using the expression for the mean deviation of the binomial random variable and summing over all $k$, we obtain the desired result. 

\end{document}